\documentclass{article}
\usepackage[utf8]{inputenc}

\newcommand{\sE}[0]{\mathscr{E}}
\newcommand{\E}[0]{\mathbb{E}}

\newcommand{\N}[0]{\mathbb{N}}

\newcommand{\Pj}[0]{\mathbb{P}}

\newcommand{\R}[0]{\mathbb{R}}
\newcommand{\cR}[0]{\mathcal{R}}

\newcommand{\one}[0]{\mathbbm{1}}

\newcommand{\al}[0]{\alpha}
\newcommand{\be}[0]{\beta}
\newcommand{\ga}[0]{\gamma}

\newcommand{\de}[0]{\delta}

\newcommand{\ep}[0]{\varepsilon}

\newcommand{\ph}[0]{\varphi}

\newcommand{\te}[0]{\theta}
\newcommand{\Te}[0]{\Theta}

\newcommand{\Om}[0]{\Omega}
\newcommand{\si}[0]{\sigma}
\newcommand{\Si}[0]{\Sigma}

\newcommand{\nin}[0]{\not\in}

\newcommand{\subeq}[0]{\subseteq}

\newcommand{\iy}[0]{\infty}

\newcommand{\rc}[1]{\frac{1}{#1}}
\newcommand{\prc}[1]{\pa{\rc{#1}}}
\newcommand{\ff}[2]{\left\lfloor\frac{#1}{#2}\right\rfloor}

\newcommand{\fc}[2]{\frac{#1}{#2}}
\newcommand{\sfc}[2]{\sqrt{\frac{#1}{#2}}}
\newcommand{\pf}[2]{\pa{\frac{#1}{#2}}}

\newcommand{\pdd}[1]{\frac{\partial}{\partial #1}}
\newcommand{\ddd}[1]{\frac{d}{d #1}}

\newcommand{\pl}[0]{\partial}
\newcommand{\nb}[0]{\nabla}

\newcommand{\dy}{\,dy}
\newcommand{\dx}{\,dx}

\newcommand{\ab}[1]{\left| {#1} \right|}
\newcommand{\an}[1]{\left\langle {#1}\right\rangle}
\newcommand{\ba}[1]{\left[ {#1} \right]}
\newcommand{\bc}[1]{\left\{ {#1} \right\}}

\newcommand{\pa}[1]{\left( {#1} \right)}

\newcommand{\ve}[1]{\left\Vert {#1}\right\Vert}

\newcommand{\ol}[1]{\overline{#1}}

\newcommand{\wt}[1]{\widetilde{#1}}
\newcommand{\wh}[1]{\widehat{#1}}

\newcommand{\amax}{\operatorname{argmax}}
\newcommand{\amin}{\operatorname{argmin}}

\newcommand{\KL}[0]{\operatorname{KL}}

\newcommand{\TV}[0]{\operatorname{TV}}

\newcommand{\Var}[0]{\operatorname{Var}}

\providecommand{\cal}[1]{\mathcal{#1}}
\renewcommand{\cal}[1]{\mathcal{#1}}

\newcommand{\pull}[9]{
#1\ar@/_/[ddr]_{#2} \ar@{.>}[rd]^{#3} \ar@/^/[rrd]^{#4} & &\\
& #5\ar[r]^{#6}\ar[d]^{#8} &#7\ar[d]^{#9} \\}

\newcommand{\cmp}[9]{
\xymatrix{
#1 \ar[r]^{#4}{#5} \ar@/_2pc/[rr]^{#8}_{#9} & #2 \ar[r]^{#6}_{#7} & #3
}
}

\newcommand{\ha}[1]{\ar@{^(->}[#1]}
\newcommand{\ls}[1]{\ar@{-}[#1]}
\newcommand{\sj}[1]{\ar@{->>}[#1]}
\newcommand{\aq}[1]{\ar@{=}[#1]}
\newcommand{\acir}[1]{\ar@{}[#1]|-{\textstyle{\circlearrowright}}}
\newcommand{\acil}[1]{\ar@{}[#1]|-{\textstyle{\circlearrowleft}}}
\newcommand{\ard}[1]{\ar@{.>}[#1]}
\newcommand{\mt}[1]{\ar@{|->}[#1]}
\newcommand{\inm}[1]{\ar@{}[#1]|-{\in}}
\newcommand{\inr}{\ar@{}[d]|-{\rotatebox[origin=c]{-90}{$\in$}}}
\newcommand{\inl}{\ar@{}[u]|-{\rotatebox[origin=c]{90}{$\in$}}}

\newcommand{\sumo}[2]{\sum_{#1=1}^{#2}}
\newcommand{\sumz}[2]{\sum_{#1=0}^{#2}}

\newcommand{\iny}[0]{\int_{-\infty}^{\infty}}

\newcommand{\beq}[1]{\begin{equation}\llabel{#1}}
\newcommand{\eeq}[0]{\end{equation}}
\newcommand{\bal}[0]{\begin{align*}}
\newcommand{\eal}[0]{\end{align*}}
\newcommand{\ban}[0]{\begin{align}}
\newcommand{\ean}[0]{\end{align}}

\newcommand{\blu}[1]{{\color{blue}#1}}

\newcommand{\fixme}[1]{#1}
\newcommand{\llabel}[1]{\label{#1}\text{\fixme{\tiny#1}}}

\newcommand{\arxiv}[1]{\url{http://www.arxiv.org/abs/#1}}

\DeclareFontFamily{U}{wncy}{}
    \DeclareFontShape{U}{wncy}{m}{n}{<->wncyr10}{}
    \DeclareSymbolFont{mcy}{U}{wncy}{m}{n}
    \DeclareMathSymbol{\Sh}{\mathord}{mcy}{"58}

\usepackage{fullpage}

\usepackage{amsmath}
\usepackage{amssymb}
\usepackage{amsthm}
\usepackage{mathrsfs}
\usepackage{bbm}
\usepackage{cancel}
\usepackage{hyperref}
\usepackage{thm-restate}
\usepackage{xypic}

\usepackage{algorithm}
\usepackage{algorithmicx,algpseudocode}

\usepackage{color}

\usepackage[backend=biber,style=alphabetic,maxbibnames=99]{biblatex}
\newtheorem{thm}{Theorem}[section]
\newtheorem*{thm*}{Theorem}

\newtheorem*{clm*}{Claim}

\newtheorem*{conj*}{Conjecture}
\newtheorem{cor}[thm]{Corollary}
\newtheorem{lem}[thm]{Lemma}
\newtheorem*{lem*}{Lemma}

\newtheorem*{prb*}{Problem}

\newtheorem{asm}{Assumption}

\newtheorem*{ax*}{Axiom}

\newtheorem*{df*}{Definition}

\newtheorem*{ex*}{Example}

\newtheorem*{pos*}{Postulate}

\newtheorem*{pr*}{Proposition}

\newtheorem*{qu*}{Question}

\newtheorem*{rem*}{Remark}

\usepackage{authblk}
\title{Convergence for score-based generative modeling with polynomial complexity}
\author[1]{Holden Lee}
\author[2]{Jianfeng Lu}
\author[2]{Yixin Tan}
\affil[1]{Johns Hopkins University}
\affil[2]{Duke University}
\date{\today}

\begin{document}
\newcommand{\CP}[0]{C_{\textup P}}
\newcommand{\CLS}[0]{C_{\textup{LS}}}
\newcommand{\tref}[1]{\text{\ref{#1}}}
\newcommand{\pdata}[0]{p_{\textup{data}}}
\newcommand{\ppr}[0]{p_{\textup{prior}}}
\newcommand{\etv}[0]{\ep_{\TV}}
\newcommand{\echi}[0]{\ep_{\chi}}

\definecolor{purple}{rgb}{0.5, 0.0, 0.5}
\newcommand{\hlnote}[1]{\textcolor{purple}{\textbf{HL:} #1}}

\newcommand{\yt}[1]{\textcolor{red}{\textbf{YT:} #1}}

\newcommand{\jl}[1]{\textcolor{blue}{\textbf{JL:} #1}}

\renewcommand{\cR}[0]{\mathcal{R}}

\newcommand{\Bad}[0]{B}
\newcommand{\Good}[0]{A}
\newcommand{\smf}[0]{L}
\newcommand{\smg}[0]{\beta_g} 
\newcommand{\sms}[0]{L_s}
\newcommand{\score}[0]{s}

\newcommand{\Mkh}[0]{\varepsilon_{kh}}
\newcommand{\Gkht}[0]{G_{kh,t}}

\maketitle

\begin{abstract}
    Score-based generative modeling (SGM) is a highly successful approach for learning a probability distribution from data and generating further samples.
    We prove the first polynomial convergence guarantees for the core mechanic behind SGM: drawing samples from a probability density $p$ given a score estimate (an estimate of $\nabla \ln p$) that is accurate in $L^2(p)$.
    Compared to previous works, we do not incur error that grows exponentially in time or that suffers from a curse of dimensionality. Our guarantee works for any smooth distribution and depends polynomially on its log-Sobolev constant. 
    Using our guarantee, we give a theoretical analysis of score-based generative modeling, which transforms white-noise input into samples from a learned data distribution given score estimates at different noise scales.
    Our analysis gives theoretical grounding to the observation that an annealed procedure is required in practice to generate good samples, as our proof depends essentially on using annealing to obtain a warm start at each step.
    Moreover, we show that a predictor-corrector algorithm gives better convergence than using either portion alone. 
\end{abstract}


\section{Introduction}


A key task in machine learning is to learn a probability distribution from data, in a way that allows efficient generation of additional samples from the learned distribution. 
Score-based generative modeling (SGM) is one empirically successful approach that \emph{implicitly} learns the probability distribution by learning how to transform white noise into the data distribution, and gives state-of-the-art performance for generating images and audio \cite{song2019generative, dathathri2019plug, grathwohl2019your, song2020improved, song2020score, meng2021sdedit, song2021solving, song2021maximum, jing2022subspace}. It also yields a conditional generation process for inverse problems \cite{dhariwal2021diffusion}. The basic idea behind score-based generative modeling is to first estimate the score function from data \cite{song2020sliced} and then to sample the distribution based on the learned score function. 
Other approaches for generative modeling include generative adversarial networks (GANs) \cite{goodfellow2014generative, arjovsky2017wasserstein}, normalizing flows \cite{dinh2016density}, variational autoencoders \cite{kingma2019introduction}, and energy-based models \cite{zhao2016energy}. 
While score-based generative modeling has achieved great success, its theoretical analysis is still lacking and is the focus of our work. 

\subsection{Background}

\paragraph{General framework.} 
The \emph{score function} of a distribution $P$ with density $p$ is defined as the gradient of the log-pdf, $\nb \ln p$. Its significance arises from the fact that knowing the score function allows running a variety of sampling algorithms, based on discretizations of stochastic differential equations (SDE's), to sample from $p$. SGM consists of two steps: first, learning an estimate of the score function for a sequence of ``noisy" versions of the data distribution $P_{\textup{data}}$, and second, using the score function in lieu of the gradient of the log-pdf in the chosen sampling algorithm. We now describe each of these steps more precisely.

First, a method of adding noise to the data distribution is fixed; this takes the form of evolving a (forward) stochastic differential equation (SDE) starting from the data distribution. 
We fix a sequence of noise levels $\si_1<\cdots < \si_N$. For $\si\in \{\si_1,\ldots, \si_N\}$, let the resulting distributions be $P_{\si^2}$ and the distributions conditional on the starting data point be $P_{\si^2}(\cdot|x)$.
Typically, $\si_1$ is chosen so that $P_{\si_1^2}\approx P_{\textup{data}}$ and $P_{\si_N^2}$ is close to some ``prior" distribution that is easy to sample from, such as $N(0,\si_N^2I_d)$.
While the score $\nb\ln p_{\si^2}$ cannot be estimated directly, it turns out that a de-noising objective that is equivalent to the score-matching objective can be calculated \cite{song2019generative}. 
This de-noising objective can be estimated 
from samples $(X, \wt X)$ where $\wt X \sim P_{\si^2}(\cdot|x)$.
The objective is represented and optimized within an expressive function class, typically neural networks, to obtain a $L^2$-estimate of the score, that is, $s_\te(x,\sigma^2)$ such that 
\begin{align}\label{e:score-l2}
\E_{x\sim P_{\si^2}}
[\|s_\te(x,\sigma^2) - \nb \ln p_{\si^2}(x)\|^2]
\end{align}
is small.

The reason we estimate the score function $\nb \ln p_{\si^2}$ is that there are a variety of sampling algorithms---based on simulating SDE's---that can sample from $p$ given access to $\nb \ln p$, including Langevin Monte Carlo and Hamiltonian Monte Carlo. The second step is then to use the estimated score function $s_\te(x,t)$ in lieu of the exact gradient in the sampling algorithm to successively obtain samples from $p_{\si_N^2},\ldots, p_{\si_1^2}$. This sequence interpolates smoothly between the prior distribution (e.g., $N(0,\si_N^2I_d)$) and the data distribution $P_{\textup{data}}$; such an ``annealing" or ``homotopy" method is required in practice to generate good samples \cite{song2020score}.

\paragraph{Examples of SGM's.}
There have been several instantiations of this general approach.~\cite{song2019generative} add gaussian noise to the data 
and then use Langevin diffusion at a discrete set of noise levels $\si_N>\cdots > \si_1$ as the sampling algorithm.~\cite{song2020score} take the continuous perspective and consider a more general framework, where the forward process can be any reasonable SDE. Then a natural \emph{reverse SDE} evolves the final distribution $p_{\si_N^2}$ back to the data distribution; this process can be simulated with the estimated score. They consider methods based on two different SDE's: score-matching Langevin diffusion (SMLD) based on adding Gaussian noise 
and denosing diffusion probabilistic models (DDPM)~\cite{sohl2015deep,ho2020denoising}, based on the Ornstein-Uhlenbeck process. Note that a difference with MCMC-based methods is that these SDE's are evolved for a fixed amount of time, rather than until convergence. However, they can be combined with MCMC-based methods such as Langevin diffusion in the \emph{predictor-corrector} approach for improved convergence. \cite{dockhorn2021score} include Hamiltonian dynamics: they augment the state space with a velocity variable and consider a critically-damped version of the Ornstein-Uhlenbeck process. Finally, we note the work of~\cite{de2021diffusion}, who introduce the Diffusion Schr\"odinger Bridge method to learn a diffusion that more quickly transforms the prior into the data distribution.

We will give a general analysis framework for SGM's that applies to the algorithms in both~\cite{song2019generative} and~\cite{song2020score}.

\subsection{Prior work and challenges for theory}
Although the literature on convergence for Langevin Monte Carlo~\cite{durmus2017nonasymptotic, cheng2018convergence, cheng2018underdamped, dalalyan2017theoretical, dalalyan2019user, majka2020nonasymptotic, erdogdu2021convergence} 
and related sampling algorithms is extensive,
prior works mainly consider the case of exact or stochastic gradients. In contrast, by the structure of the loss function~\eqref{e:score-l2}, the score function learned in SGM is only accurate in $L^2(p)$. This poses a significant challenge for analysis, as the stationary distribution of Langevin diffusion with $L^2(p)$-accurate gradient can be arbitrarily far from $p$ (see Appendix~\ref{a:far}). Hence, any analysis must be utilizing the short/medium-term convergence, while overcoming the potential issue of long-term behavior of convergence to an incorrect distribution.

\cite{block2020generative} give the first theoretical analysis of
SGM, and in particular, Langevin Monte Carlo with $L^2(p)$-accurate gradients.
First, they show using uniform generalization bounds that optimizing the de-noising autoencoder (DAE) objective does in fact give a $L^2(p)$-accurate score function, with sample complexity depending on the complexity of the function class. They analyze convergence of LMC in Wasserstein distance. However, the error they obtain (Theorem 13) only decreases as $\ep^{1/d}$ where $\ep$ is the accuracy of the score estimate---so it suffers from the curse of dimensionality---and increases exponentially in the time that the process is run, the dimension, and the smoothness of the distribution, as in ODE/SDE discretization arguments that do not depend on contractivity.

\cite{de2021diffusion} give an analysis for~\cite{song2020score}  in TV distance that requires a $L^\iy$-accurate score function and depends exponentially on the amount of time the reverse SDE is run. Although exponential dependence is bad in general, it is mollified using their Diffusion Schr\"odinger Bridge (DSB) approach, as it allows running for a shorter, fixed amount of time, before the forward SDE converges to the prior distribution. However, this supposes that a good solution can be found for the DSB problem, and theoretical guarantees may be difficult to obtain. 


We overcome the challenges of analysis with a $L^2(p)$-accurate gradient, and give the first analysis with only polynomial dependence on running time, dimension, and smoothness of the distribution, with rates that are a fixed power of $\ep$. 
Our convergence result is in TV distance. We assume only smoothness conditions and a bounded log-Sobolev constant of the data distribution, a weaker condition than the dissipativity condition required by~\cite{block2020generative}.
We introduce a general framework for analysis of sampling algorithms given $L^2$-accurate gradients (score function) based on constructing a ``bad set'' with small measure and showing convergence of the discretized process conditioned on not hitting the bad set. We use our framework to give an end-to-end analysis for both the algorithms in~\cite{song2019generative} and~\cite{song2020score}, and illuminate the relative performance of different methods in practice.





\subsection{Notation and organization}

Through out the paper, $p(x)\propto e^{-V(x)}$ denotes the target distribution in $\mathbb R^d$ and $V:\mathbb R^d\rightarrow \mathbb R$ is referred to as the potential. We abuse notation by identifying a measure with its density when context allows. We write \fixme{$a\wedge b:=\min\{a,b\}$ and $a\vee b:=\max\{a,b\}$}. We use $a=O(b)$ or $b=\Om(a)$ to indicate that $a\leq Cb$ for a universal constant $C>0$. Also, we write $a=\Theta(b)$ if there are universal constants $c'>c>0$ such that $cb\leq a\leq cb$, and the notation $\Tilde{O}(\cdot)$ means it hides polylog factors in the parameters.
Definite integrals without limits are taken over $\R^d$.

In Section~\ref{s:lmc} we explain our main results for Langevin Monte Carlo with $L^2(p)$-accurate score estimate and use it to derive convergence bounds for the annealed LMC method of~\cite{song2019generative}. In Section~\ref{s:pc}, we give our main results for the predictor-corrector algorithms of~\cite{song2020score} based on simulating reverse SDE's.
Our proofs are based on a common framework which we introduce in Section~\ref{s:sketch}. Full proofs are in the appendix.


\section{Results for Langevin dynamics with estimated score}
\label{s:lmc}

Let $p(x)\propto e^{-V(x)}$ be a probability density on $\R^d$ such that $V$ is $C^1$. 
Langevin diffusion with stationary distribution $p$ is the stochastic process defined by the SDE
\[
dx_t = -\nb V(x_t) \,dt + \sqrt 2 \,dw_t,
\]
where $w_t$ is a standard Brownian Motion in $\R^d$. The rate of convergence to $p$ in $\chi^2$ and KL divergences are given by the Poincar\'e and log-Sobolev constants of $p$, respectively; see Section~\ref{s:facts}. 
To obtain the Langevin Monte Carlo (LMC) algorithm, we take the Euler-Murayama discretization of the SDE. We define LMC with score estimate $s(x)\approx -\nb V(x)$ and step size $h$ by
\begin{align}
\label{e:lmc-se}
    x_{(k+1)h} &= x_{kh} + h\cdot s(x_{kh}) + \sqrt{2h} \cdot \xi_{kh}, \text{ where }\xi_{kh}\sim N(0,I_d).
    \tag{LMC-SE}
\end{align}
We make the following assumptions on the density $p$ and the score estimate $s$, which we will  use throughout this paper. 

\begin{asm}\label{a:p}
$p$ is a probability density on $\R^d$ such that the following hold. 
\begin{enumerate}
    \item $\ln p$ is $C^1$ and $L$-smooth, that is, $\nb \ln p$ is $L$-Lipschitz. We assume $L\ge 1$.
    \label{a:p-smooth}
    \item $p$ satisfies a log-Sobolev inequality with constant $\CLS$. We assume $\CLS\ge 1$.
    \label{a:p-lsi}
    \item 
    \fixme{(Moments)
    \label{a:p-mean}
   $\ve{\E_{p} x}\le M_1$ and $\E_{p} \ve{x}^2\le M_2$.}
\end{enumerate}
\end{asm}
We note that the uniform Lipschitzness assumption (1) helps ensure a unique strong solution to the Langevin diffusion, as in~\cite{block2020generative}. One special case where one can prove Lipschitzness for all $t$ is when $p_0$ is strongly log-concave~\cite[Lemma 28]{lee2021universal}. 
Although satisfying a log-Sobolev inequality (3) is a significant assumption, it is standard for analysis of Langevin Monte Carlo~\cite{vempala2019rapid}. It is much weaker than assumptions in previous works~\cite{block2020generative}, including log-concave distributions and distributions satisfying strong dissipativity, and is stable under bounded perturbations. See Section~\ref{s:facts} for background on functional inequalities.

\begin{asm}\label{a:score}
Let $p$ be a given probability density on $\R^d$ such that $\ln p$ is $C^1$. The score estimate $s:\R^d\to \R^d$ satisfies the following.
\begin{enumerate}
    \item $s$ is a $C^1$ function that is $\sms$-Lipschitz. We assume $\sms\ge 1$.
    \label{a:score-smooth}
    \item The error in the score estimate is bounded in $L^2$:
    \[
    \ve{\nb \ln p-s}_{L^2(p)}^2=
    \E_{p}[\ve{\nb \ln p(x) - s(x)}^2]\le \ep^2.
    \]
    \label{a:score-error}
\end{enumerate}
\end{asm}
\subsection{Langevin with $L^2$-accurate score estimate}

Our first main result gives an error bound between the sampled distribution and $p$, assuming $L^2$-accurate score function estimate.

\begin{restatable}[LMC with $L^2$-accurate score estimate]{thm}{correctorthm}\label{t:corrector-tv-chi2}
Let $p:\R^d\to \R$ be a  probability density satisfying Assumption~\ref{a:p}(\ref{a:p-smooth}, \ref{a:p-lsi}) with $\smf \ge 1$ and $s:\R^d\to \R^d$ be a score estimate  satisfying Assumption~\ref{a:score}(\ref{a:score-error}). 
Consider the accuracy requirement in $\TV$ and $\chi^2$:
$0<\ep_{\textup{TV}}<1$,
$0<\ep_{\chi}<1$, 
and suppose furthermore the starting distribution satisfies $\chi^2(p_0||p)\le K_\chi^2$. 
Then if
\begin{align}
\label{e:lmc-ep}
\ep &= 
    O\pf{\ep_{\TV}\ep_\chi^3}{d\smf^2\CLS^{5/2} (\ln (2K_\chi/\ep_\chi^2)\vee K_\chi)},
\end{align}
then running 
\eqref{e:lmc-se}
with score estimate $s$, step size 
$
h=\Theta\Bigl(\fc{\ep_\chi^2}{d\smf^2 \CLS}\Bigr),
$
and time $T=\Te\Bigl(\CLS \ln\bigl(\frac{2K_\chi}{\ep_\chi^2}\bigr)\Bigr)$ 
results in a distribution $p_T$ such that 
$p_T$ is $\ep_{\TV}$-far in TV distance from a distribution $\ol p_T$, 
where $\ol p_T$ satisfies
    $
        \chi^2(\ol p_T || p) \le \ep_\chi^2.
    $
In particular, taking $\ep_\chi=\ep_{\TV}$, we have the error guarantee that $\TV(p_T,p)\le 2\ep_{\TV}$.
\end{restatable}
Note that the error bound is only achieved when running LMC for a moderate time; this is consistent with the fact that the stationary distribution of LMC with a $L^2$-score estimate can be arbitrarily far from $p$. Note also that we need a warm start in $\chi^2$-divergence: to obtain fixed errors $\etv,\echi$, the required accuracy for the score estimate is inversely proportional to $K_\chi$.
Intuitively, we must suffer from such a dependence because if the starting distribution is very far away, then there is no guarantee that $\|\fixme{\nb}\ln p(x_t) - s(x_t)\|^2$ is small on average during the sampling algorithm. Finally, although we can state a result purely in terms of TV distance, we need this more precise formulation to prove a result for annealed Langevin dynamics.

\subsection{Annealed Langevin dynamics with estimated score}\label{sec:annealed}

In light of the warm start requirement in Theorem~\ref{t:corrector-tv-chi2}, we typically cannot directly sample from $p_{\textup{data}}$ or its approximation. Hence, \cite{song2019generative} proposed using annealed Langevin dynamics: consider a sequence of noise levels $\si_N>\cdots >\si_1\approx 0$ giving rise to a sequence of distributions $p_{\si_N^2},\ldots, p_{\si_1^2}\approx p_{\textup{data}}$, where $p_{\si^2} = p * \ph_{\si^2}$, $\ph_{\si^2}$ being the density of $N(0,\si^2I_d)$. For large enough $\si_N$, $\ph_{\si_N^2}\approx p_{\si_N^2}$ provides a warm start to $p_{\si_N^2}$. We then successively run LMC using score estimates for $p_{\si_k^2}$, with the approximate sample for $p_{\si_k^2}$ giving a warm start for $p_{\si_{k-1}^2}$. We obtain the following algorithm and error estimate.

\begin{algorithm}[h!]
\begin{algorithmic}
\State INPUT: Noise levels $0\le \si_1<\ldots< \si_M$; score function estimates $s(\cdot, \si_m)$ (estimates of $\nb\ln (p*\ph_{\si_m^2})$), step sizes $h_m$, and number of steps $N_m$ for $1\le m\le M$.
\State Draw $x^{(M+1)}\sim N(0,\si_M^2I_d)$. 
\For{$m$ from $M$ to 1}
	\State Starting from $x^{(m)}_0=x^{(m+1)}$, run~\eqref{e:lmc-se} with $s(x,\si_m)$ and step size $h_m$ for $N_m$ steps, and let the final sample be $x^{(m)}$.
\EndFor
\State OUTPUT: Return $x^{(1)}$, approximate sample from $p*\ph_{\si_1^2}$.
\end{algorithmic}
 \caption{Annealed Langevin dynamics with estimated score~\cite{song2019generative}}
 \label{a:ald}
\end{algorithm}

\begin{restatable}[Annealed LMC with $L^2$-accurate score estimate]{thm}{aldtheorem}
\label{t:ald}
Let $p:\R^d\to \R$ be a  probability density satisfying Assumption~\ref{a:p} for $M_1=O(d)$, 
and let $p_{\si^2}:=p*\ph_{\si^2}$.
Suppose furthermore that $\nb\ln p_{\si^2}$ is $L$-Lipschitz for every $\si\ge 0$. 
Given $\si_{\min}>0$, there exists a sequence $\si_{\min}=\si_1<\cdots < \si_M$ with $M=O\pa{\sqrt d\log \pf{d\CLS}{\si_{\min}^2}}$ such that for each $m$, if
\begin{gather}
\nonumber
    \ve{\nb \ln (p_{\si_m^2})-s(\cdot,\si_m^2)}_{L^2(p_{\si_m^2})}^2=
    \E_{p_{\si_m^2}}[\ve{\nb \ln p_{\si_m^2}(x) - s(x, \si_m^2)}^2]\le \ep^2.\\
    \text{with }
    \ep:=
    \fixme{
    \wt O\pf{\etv^{4.5}}{d^{3.25}L^2\CLS^{2.5}}
    } 
\label{e:ep-almc}
\end{gather}
then $x^{(1)}$ is a sample from a distribution $q$ 
such that $\TV(q,p_{\si_1^2})\le \etv$. 
\end{restatable}
Note that we assume a score estimate with error $\ep$ at all noise scales; this corresponds to using an objective function that is a maximum of the score-matching objective over all noise levels, rather than an average over all noise levels as more commonly used in practice. However, these two losses are at most a factor of $M$ apart.

The proof shows that the noise levels $\si_k$ can be chosen as a geometric sequence, which matches the choice used in practice~\cite{song2020improved}.
The additional dependence on $d$ and $\ep_{\TV}$ in Theorem~\ref{t:ald} compared to Theorem~\ref{t:corrector-tv-chi2}
comes from requiring a sequence of $\wt O(\sqrt d)$ noise levels and an additional factor in $\chi^2$-divergence we suffer at the beginning of each level $m$. In the next section, we will find that using a reverse SDE to evolve the samples between the noise levels---called a \fixme{\emph{predictor}} step---will improve the rate and time complexity.

\section{Results for reverse SDE's with estimated score}
\label{s:pc}

To improve the empirical performance of score-based generative modeling, 
\cite{song2020score} consider a general framework where noise is injected into a data distribution $ p_{\textup{data}}$ via a forward SDE,
\begin{align*}
    d\Tilde{x}_t = f(\Tilde{x}_t, t)\,dt +g(t)\,dw_t,\ \ t\in[0,T],
\end{align*}
where $\wt x_0\sim \wt p_0:=p_{\textup{data}}$. Let $\wt p_t$ denote the distribution of $\wt x_t$ ($\wt p_t$ is used instead of $p_t$ to distinguish with the Gaussian-convolved distribution used in Annealed Langevin dynamics as in \S\ref{sec:annealed}). Remarkably, $\wt x_t$ also satisfies a reverse-time SDE,
\begin{equation}\label{general_reverse_sde}
    d\Tilde{x}_t = [f(\Tilde{x}_t,t) - g(t)^2\nb\ln \Tilde{p}_t(\Tilde{x}_t)]dt + g(t)\,d\Tilde{w}_t,\ t\in[0, T],
\end{equation}
where $\Tilde{w}_t$ is a backward Brownian Motion~\cite{anderson1982reverse}. By carefully choosing $f$ and $g$, we can expect that $\Tilde{p}_T$ is approximately equal to some prior distribution $\Tilde{q}_T$ (e.g., a centered Gaussian) which we can accurately sample from.
Then we hope that starting with some $\Tilde{y}_T\sim \ppr = \Tilde{q}_T\approx \Tilde{p}_T$ and running the reverse-time process, we will get a good sample $\Tilde{y}_0\sim\Tilde{q}_0\approx p_{\textup{data}}$.

The case where $f\equiv 0$ and $g \equiv 1$ recovers the simple case of convolving with a Gaussian as used in \S\ref{sec:annealed}; note, however that the reverse-time SDE differs from Langevin diffusion in having a larger (and time-varying) drift relative to the diffusion.
\cite{song2020score} highlight the following two special cases. 
We will focus on 
DDPM while noting that our analysis applies more generically. 
\begin{enumerate}
    \item[SMLD]
    \label{i:SMLD}
    \textbf{Score-matching Langevin diffusion}: 
    $f\equiv 0$. In this case, $\wt p_t = \wt p_0 * \ph_{\int_0^t g(s)^2\,ds}$, so 
    \cite{song2020score} call this a variance-exploding (VE) SDE. As is common for annealing-based algorithms,~\cite{song2019generative,song2020score} suggest choosing an exponential schedule, so that $g(t) = ab^t$ for constants $a,b$. We take $\ppr = N(0,\int_0^T g(s)^2\,ds \cdot I_d)$.
    \item[DDPM]
    \label{i:DDPM}
    \textbf{Denoising diffusion probabilistic modeling}: $f(x,t)=-\fc12 g(t)^2x$. This is an Ornstein-Uhlenbeck process with time rescaling, $\wt p_t = M_{-\rc 2\int_0^t g(s)^2\,ds\sharp}\wt p_0 * \ph_{1-e^{-\int_0^tg(s)^2\,ds}}$, where $M_{\al}(x)=\al x$. \cite{song2020score} call this a variance-preserving (VP) SDE, as the variance converges towards $I_d$. Because it displays exponential convergence towards $N(0,I_d)$, it can be run for a smaller amount of normalized time $\int_0^tg(s)^2\,ds$. \cite{song2020score} suggest the choice $g(t) = \sqrt{b+\al t}$. We take $\ppr = N(0, (1-e^{-\int_0^tg(s)^2\,ds}) I_d) \approx N(0,I_d)$.
\end{enumerate}
To obtain an algorithm, we consider the following discretization and approximation of~\eqref{general_reverse_sde}; note that in all cases of interest the integrals can be analytically evaluated. We reverse time so that $t$ corresponds to $T-t$ of the forward process. As we are free to rescale time in the SDE, we assume without loss of generality that the step sizes are constant. The predictor step is
\begin{multline}\tag{P}
\label{e:P}
    z_{(k+1)h} = z_{kh} - \int_{kh}^{(k+1)h} \ba{f(z_{kh}, T-t) - g(T-t)^2\cdot s(z_{kh}, T-kh)}dt \\ + \int_{kh}^{(k+1)h}g(T-t)\,dw_t,
\end{multline}
where $\int_{kh}^{(k+1)h}g(T-t)\,dw_t$ is distributed as $N(0,\int_{kh}^{(k+1)h}g(T-t)^2\,dt \cdot I_d)$. Following~\cite{song2020score}, we call these predictor steps as the samples aim to track the distributions $\wt p_{T-kh}$. Note that we flip the time.
For simplicity of presentation, we consider the case $g\equiv 1$. 
We note that although the choice of the schedule does matter in practice, what really matters in our theoretical analysis is the integral $\int_0^t g(s)^2 ds$. This means that different choices of $g$ are related by only a rescaling of time, i.e., for different $g$ and $\tilde g$, we can always choose total times $T$ and $\tilde T$, such that $\int_0^T g(s)^2 ds = \int_0^ {\tilde {T}} \tilde g(s)^2 ds$. While it seems that choosing large $g(t)$ could reduce the total time $T$, in our analysis (e.g., Lemma~\ref{zt_zkh_difference}) we need the time step-size $h$ to be $O(1/g(T)^2)$ and hence the total computational cost, which is roughly $O(T/h)$, does not change significantly. 



\begin{thm}[Predictor with $L^2$-accurate score estimate, DDPM]
\label{t:p}
Let $p_{\textup{data}}:\R^d\to \R$ be a  probability density satisfying Assumption~\ref{a:p} with $M_2=O(d)$, 
and let $\wt p_t$ be the distribution resulting from evolving the forward SDE according to
\hyperref[i:DDPM]{DDPM} with $g\equiv 1$.
Suppose furthermore that $\nb\ln \wt p_t$ is $L$-Lipschitz for every $t\ge 0$, and that each $s(\cdot, t)$ satisfies Assumption~\ref{a:score}. 
Then if
\[
\ep = O\pf{\ep_{\TV}^4}{(\CLS+d)\CLS^{5/2} (\smf\vee \sms)^2
(\ln(\CLS d) \vee \CLS \ln(1/\etv^2))
},
\]
running~\eqref{e:P} starting from $\ppr$ for time $T=\Te\pa{\ln(\CLS d) \vee \CLS \ln\prc{\etv}}$
and step size $h=\Te\pf{\etv^2}{\CLS(\CLS+d)(\smf\vee \sms)^2}$ 
results in a distribution $q_T$ so that 
$\TV(q_T, p_{\textup{data}})\le \ep_{\TV}$.
\end{thm}

A more precise statement of the Theorem can be found in the Appendix. 
Although we state our theorem for 
DDPM, we describe in Appendix~\ref{s:pc-proof} how it can be 
adapted to other SDE's like SMLD and the sub-VP SDE; 
the primary SDE-dependent bound we need is a bound on $\nb \ln \fc{\wt p_t}{\wt p_{t+h}}$. 
Because the predictor is tracking a changing distribution $p_t$, we incur more error terms and worse dependence on parameters ($\CLS, \smf$) than in LMC (Theorem~\ref{t:corrector-tv-chi2}). Motivated by this, 
we intersperse the predictor steps with LMC steps---called \emph{corrector} steps in this context---to give additional time for the process to mix, resulting in improved dependence on parameters.

\begin{algorithm}[h!]
\begin{algorithmic}
\State INPUT: Time $T$, predictor step size $h$; number of corrector steps $N_{\fixme{m}}$ per predictor step, corrector step sizes $h_m$
\State Draw $z_0\sim \ppr$ from the prior distribution. 
\For{$m$ from $1$ to $T/h$} 
    \State (Predictor) Take a step of~\eqref{e:P} to obtain $z_{mh}$ from $z_{(m-1)h}$, with $f,g$ as in~\hyperref[i:SMLD]{SMLD} or~\hyperref[i:DDPM]{DDPM}.
	\State (Corrector) Starting from $z_{mh,0}:=z_{mh}$, run~\eqref{e:lmc-se} with $s(z,T-mh)$ and step size $h_m$ for $N$ steps, and let $z_{mh}\leftarrow z_{mh,N}$. 
\EndFor
\State OUTPUT: Return $z_T$, approximate sample from $p_{\textup{data}}$.
\end{algorithmic}
 \caption{Predictor-corrector method with estimated score~\cite{song2020score}}
 \label{a:pc}
\end{algorithm}

\begin{restatable}[Predictor-corrector with $L^2$-accurate score estimate]{thm}{pctheorem}
\label{t:pc}
Keep the setup of Theorem~\ref{t:p}. 
Then for 
\fixme{$\etv^3 = O\pf{1}{(1+\sms/\smf)^2(1+\CLS/d)(\ln (\CLS d)\vee \CLS)}$}, if
\begin{align}
    \ep  = O\pf{\ep_{\TV}^4}{dL^2\CLS^{5/2} 
\ln (1/\echi^2)
},
\end{align}
then Algorithm~\ref{a:pc} with appropriate choices of $T=\Te\pa{\ln(\CLS d) \vee \CLS \log\prc{\etv}}$, $N_{\fixme{m}}$, corrector step sizes $h_m$ and predictor step size $h$, 
produces a sample from a distribution $q_T$ such that $\TV(q_T,\pdata)<\etv$.  
\end{restatable}
The assumption on $\etv$ is for convenience in stating our bound.
In comparison to using the predictor step alone (Theorem~\ref{t:p}), note that in the bound on $\ep$, we obtain the improved rate of the corrector step as in Theorem~\ref{t:corrector-tv-chi2}; \fixme{this is because 
the predictor step only needs to track the actual distribution in $\chi^2$-divergence with error $O(1)$, and the final corrector steps are responsible for decreasing the error to $\etv$}. 
In comparison to the Annealed Langevin sampler (Algorithm~\ref{a:ald}, Theorem~\ref{t:ald}), which can be viewed as using the corrector step alone, adding a predictor step provides a better warm start for the distribution at the next smaller noise level, resulting in better dependence on parameters. 
Thus the predictor-corrector algorithm combines the strengths of the predictor and corrector steps. For real-world data, it can be challenging to estimate TV-distance between distributions given only samples, and hence difficult to check consistency with empirical observations. However, our claim that using a corrector can improve the convergence rate of DDPM/SMLD is consistent with the simulation results in Section 4.2 of~\cite{song2020score}.


\section{Theoretical framework and proof sketches}
\label{s:sketch}

The main idea of our analysis framework is to convert a $L^2$ error guarantee to a $L^{\infty}$ error guarantee by excluding a bad set, formalized in the following theorem. 

\begin{thm}\label{t:framework}
Let $(\Omega, \cal F, \Pj)$ be a probability space and $\{\cal F_n\}$ be a filtration of the sigma field $\cal F$. Suppose $X_n\sim p_n$, $Z_n\sim q_n$, and $\ol Z_n\sim \ol q_n$ are $\cal F_n$-adapted random processes taking values in $\Om$, and $B_n\subeq \Om$ are sets such that the following hold for every $n\in \N_0$.
\begin{enumerate}
    \item If $Z_k \in B_k^c$ for all $0\le k\le n-1$, then $Z_n=\ol Z_n$. (For $n=0$, this says $Z_0=\ol Z_0$.)
    \label{i:couple}
    \item $\chi^2(\ol q_n||p_n)\le D_n^2$.
    \item 
    $\Pj(X_n\in B_n)\le \de_n$.
\end{enumerate}
Then the following hold.
\begin{align}
\TV(q_n, \ol q_n)&\le \sumz k{n-1} (D_k^2+1)^{1/2}\de_k^{1/2}&
\TV(p_n, q_n) &\le D_n+ \sumz k{n-1} (D_k^2+1)^{1/2}\de_k^{1/2} 
\end{align}
\end{thm}
For our setting, we will take the ``bad sets" $B_n$ to be the set of $x$ where $\ve{s_\te(x) - \nb \ln p}$ is large, 
$q_n$ to be the discretized process with estimated score, and $\ol q_n$ to be the discretized process with estimated score except in $B_n$ where the error is large. Because $\ol q_n$ uses an $L^\iy$-accurate score estimate, we can use existing techniques for analyzing Langevin Monte Carlo~\cite{vempala2019rapid,erdogdu2021convergence,chewi2021analysis} to bound $\chi^2(\ol q_n||p_n)$. 
\begin{proof}
First note that if some $Z_k\in B_k$ for $0\le k\le n-1$, then for the smallest such $k$, we have $\ol Z_k = Z_k\in B_k$; the same is true if $\ol Z_k\in B_k$ for some $0\le k\le n-1$. 
We then bound using condition 1 and  Cauchy-Schwarz:
\begin{align*}
\Pj\pa{Z_n \ne \ol Z_n} &\le 
\Pj\pa{\bigcup_{k=0}^{n-1}\bc{Z_k\in B_k}}
= \Pj \pa{\bigcup_{k=0}^{n-1}\{\ol Z_k\in B_k\}} 
\\
&\le 
    \sumz k{n-1} 
\Pj\pa{\ol Z_k\in B_k} 
= \sumz k{n-1}
\E_{\ol q_k} \one_{B_k}
\\
&\le \sumz k{n-1} \pa{\E_{p_k} \pf{\ol q_k}{p_k}^2}^{1/2} \pa{\E_{p_k}\one_{B_k}}^{1/2} 
= \sumz k{n-1} (D_k^2+1)^{1/2}\de_k^{1/2}.
\end{align*}
The second inequality then follows from the triangle inequality and Cauchy-Schwarz:
\begin{align*}
    \TV(p_n, q_n)&\le 
    \TV(p_n, \ol q_n) + \TV(\ol q_n, q_n)\\ 
    &\le 
    \sqrt{\chi^2(\ol q_n||p_n)} + \TV(\ol q_n, q_n) 
    \le 
    D_n + 
    \sumz k{n-1} (D_k^2+1)^{1/2}\de_k^{1/2}.
    \qedhere 
\end{align*}
\end{proof}

It now remains to give $\chi^2$ convergence bounds under $L^\iy$-accurate score estimate. The following theorem may be of independent interest.

\begin{restatable}[LMC under $L^\iy$ bound on gradient error]{thm}{lmcliy}\label{t:lmc-iy}
Let $p:\R^d\to \R$ be a  probability density satisfying Assumption~\ref{a:p}(\ref{a:p-smooth}, \ref{a:p-lsi}) and $s:\R^d\to \R^d$ be a score estimate $s$ 
with error bounded in $L^\iy$: for some $\ep_1\le \sfc{1}{48\CLS}$,
    \[
    \ve{\nb \ln p - s}_\iy = 
    \max_{x\in \R^d} \ve{\nb \ln p(x) - s(x)}]\le \ep_1.\]
Let $N\in \N_0$ and $0<h\le \rc{4392d\CLS\smf^2}$, and assume $\smf\ge 1$. Let $q_{nh}$ denote the $n$th iterate of LMC with step size $h$ score estimate $s$. Then 
\begin{align*}
   \chi^2(q_{(k+1)h}||p)&\le 
    \exp\pa{-\fc{h}{4\CLS}}\chi^2(q_{kh}||p) +
    170 d\smf^2h^2 + 5\ep_1^2h 
\end{align*}
and
\begin{align*}
    \chi^2(q_{Nh}||p) &\le \exp\pa{-\fc{Nh}{4\CLS}}\chi^2(q_0||p) + 
    680 d\smf^2h\CLS + 20\ep_1^2 \CLS
    \le \exp\pa{-\fc{Nh}{4\CLS}}\chi^2(q_0||p) + 1
\end{align*}
\end{restatable}
Following~\cite{chewi2021analysis}, we prove this by first defining a continuous-time interpolation $q_t$ of the discrete process, and then deriving a differential inequality for $\chi^2(q_t||p)$ using the log-Sobolev inequality for $p$. Compared to~\cite{chewi2021analysis}, we incur an extra error term arising from the inaccurate gradient.

This allows us to sketch the proof of Theorem~\ref{t:corrector-tv-chi2}; a complete proof is in Section~\ref{s:c}.
\begin{proof}[Proof sketch of Theorem~\ref{t:corrector-tv-chi2}]
We first define the bad set where the error in the score estimate is large,
\begin{align*}
B:&=\bc{\ve{\nb \ln p(x)-s(x)}>\ep_1}
\end{align*}
for some $\ep_1$ to be chosen. Then by Chebyshev's inequality, $P(B) \le \pf{\ep}{\ep_1}^2=:\de$. Let $\ol q_{nh}$ be the discretized process, but where the score estimate is set to be equal to $\nb \ln p$ on $B$; note it agrees with $q_{nh}$ as long as it has not hit $B$.
Because $\ol q_{nh}$ uses a score estimate that has $L^\iy$-error $\ep_1$, Theorem~\ref{t:lmc-iy} gives a bound for $\chi^2(\ol q_{Nh}||p)$. 
Then Theorem~\ref{t:framework} gives
\begin{align*}
\TV(q_{nh},\ol q_{nh}) 
&\le  
\sum_{k=0}^{n-1} (\chi^2(\ol q_{kh}||p)+1)^{1/2}P(B)^{1/2}\le 
\sumz k{n-1} \pa{\exp\pa{-\fc{kh}{8\CLS}}\chi^2(q_0||p)^{1/2} + 1} \de^{1/2}
\end{align*}
The theorem then follows from choosing parameters so that $\chi^2(\ol q_{T}||p)\le \ep_\chi^2$ and $\TV(q_{T},\ol q_{T})\le \ep_{\TV}$. \end{proof}
We remark that the main inefficiency in the proof comes from the use of Chebyshev's inequality, and a $L^p$ bound on the error for $p>2$ will improve the bound.

\begin{proof}[Proof sketch of Theorem~\ref{t:ald}]
Choosing the sequence $\si_1<\cdots<\si_M$ to be geometric with ratio $1+\rc{\sqrt{d}}$ ensures that the $\chi^2$-divergence between successive distributions $p_{\si_m^2}$ is $O(1)$. 
Then, choosing $\si_M^2=\Om(\CLS d)$ ensures we have a warm start for the highest noise level: $\chi^2(\ppr||p_{\si_M^2})=O(1)$. This uses $O\pa{\sqrt d\log \pf{d\CLS}{\si_{\min}^2}}$ noise levels. Chebyshev's inequality can be used to show that the distribution of the final sample $x^{(m)}$ for $p_{\si_m^2}$ is $O(\etv/M)$ close to a distribution that is $O(M/\etv)$ in $\chi^2$-divergence from $p_{\si_{m+1}^2}$. This gives the warm start parameter $K_\chi = (M/\etv)^{1/2}$; substituting into Theorem~\ref{t:corrector-tv-chi2} then gives the required bound for $\ep$. Note that the TV errors accrued from each level add to $O(\etv)$.
\end{proof}

To analyze the predictor-based algorithms, we also first prove convergence bounds under $L^\iy$-accurate score estimate.
\begin{thm}[Predictor steps under $L^\iy$ bound on score estimate, DDPM]
\label{t:p-iy-simple}
Let $p:\R^d\to \R$ be a  probability density satisfying Assumption~\ref{a:p} 
and $s(\cdot, t):\R^d\to \R^d$ be a score estimate $s$ 
with error bounded in $L^\iy$ for each $t\in [0,T]$: 
    \[
    \ve{\nb \ln p - s(\cdot, t)}_\iy = 
    \max_{x\in \R^d} \ve{\nb \ln \wt p_t(x) - s(x, t)}]\le \ep_1.\]
Consider \hyperref[i:DDPM]{DDPM} with $g\equiv 1$, $T\ge 1\vee \ln(\CLS d)$, and 
\fixme{$h=O\prc{\CLS (d+\CLS)(\smf\vee\sms)^2}$}. (Recall that $p_{kh}$ and $q_{kh}$ are the $k$-th iterate of LMC with step size $h$ and true/estimated score respectively.)
Then
    \begin{align*}
    \chi^2(q_{(k+1)h}||p_{(k+1)h}) \leq
    \chi^2(q_{kh}||p_{kh}) e^{\pa{-\rc{8\CLS} + 8\ep_1^2}\fixme{h}} + O(\ep_1^2 h + 
    \fixme{(\sms^2 + \smf^2d)}h^2
    )
    \end{align*}
    and if $\ep_1 <\rc{128\CLS}$,
    \begin{align*}
        \chi^2(q_{Nh}||p_{Nh})
        &\le 
        e^{-\fc{Nh}{16\CLS}} \chi^2(q_0||p_0) + O\pa{\CLS\pa{\ep_1^2 + 
        \fixme{(\sms^2 + \smf^2d)}
        h}}.
    \end{align*}
    Moreover, for $q_0 = \ppr$, $\chi^2(q_0||p_0)\le e^{-T/2}\CLS d$.
\end{thm}
We give a more precise statement 
in Section~\ref{s:pc-proof}.
Note that unlike the case for LMC as in Theorem~\ref{t:lmc-iy}, the
base density $p_t$ is also evolving in time, which produces additional error terms  and necessitates a more involved analysis.
The additional error terms can be bounded using the  Donsker-Varadhan variational principle, concentration for distributions satisfying LSI, and error bounds between $p_t$ and $p_{t+h}$ for small $h$.

Here, we only state the result about DDPM, which has better bounds than SMLD (when $g\equiv 1$) because both the forward and backwards processes exhibit better mixing properties: the warm start improves exponentially rather than inversely with $T$, and the log-Sobolev constant is uniformly bounded by that of $\pdata$ rather than increasing. However, the analysis in Section~\ref{s:pc-proof} can be directly applied to SMLD and other models as well. \fixme{We also note there is a sense in which DDPM and SMLD are equivalent under a rescaling in time and space (see discussion in Section~\ref{ss:p-1step}).} 

Note that the choice of $h$ is necessary for exponential decay of error; as if $h$ is not small enough, we would get an exponential growing instead of decaying factor in the one-step error (See Section~\ref{s:pc-proof} for details). Such an $h$ may however still be a suitable choice when used in conjunction with a corrector step. 
Moreover, as $\ep_1\to 0$, with appropriate choice of $T$ and $h$, $q_{Nh}$ and $p_{Nh}$ can be made arbitrarily close.

Theorem~\ref{t:p} now follows from the $L^\iy$ result (Theorem~\ref{t:p-iy-simple}) in the same way that Theorem~\ref{t:corrector-tv-chi2} follows from Theorem~\ref{t:lmc-iy}. 

\fixme{To prove Theorem~\ref{t:pc}, it suffices to run the corrector steps only at the lowest noise level, that is, set $N_m=0$ for $1\le m<T/h$, although we note that interleaving the predictor and corrector steps does empirically help with mixing. The proof follows from using the predictor and the corrector theorems in series: first apply Theorem~\ref{t:p} with $\echi=O(1)$ to show that the predictor  results a warm start $\pdata$, then use Theorem~\ref{t:corrector-tv-chi2} to show the corrector reduces the error to the desired $\etv$.
}

\section{Conclusion}


We introduced a general framework to analyze SDE-based sampling algorithms given a $L^2$-error score estimate, and used it to obtain the first convergence bounds for several score-based generative models with polynomial complexity in all parameters. 
Our analysis can potentially be adapted to other SDE's and sampling algorithms beyond Langevin Monte Carlo. There is also room for improving our analysis to better use smoothing properties of the SDE's and compare different choices of the diffusion speed $g$.

We present several interesting further directions to explore. In addition to extending the analysis to other SGM's and comparing their theoretical performance (relative to each other as well as other approaches to generative modeling), we propose the following.

\paragraph{Analysis for multimodal distributions.}
    Our assumption of a bounded log-Sobolev constant essentially limits the analysis to distributions that are close to unimodal. However, SGM's are empirically successful at modeling multimodal distributions~\cite{song2019generative}, and in fact perform better with multimodal distributions than other approaches such as GAN's. Can we analyze the convergence for simple multimodal distributions, such as a mixture of distributions each with bounded log-Sobolev constant? Positive results on sampling from multimodal distributions such as~\cite{ge2018beyond} suggest this is possible, as the sequence of noised distributions is natural for annealing and tempering methods (see
    ~\cite[Remark 7.2]{ge2018beyond}).
    
\paragraph{Weakening conditions on the score estimate.} 
The assumption that we have a score estimate that is $O(1)$-accurate in $L^2$, although weaker than the usual assumptions for theoretical analysis, is in fact still a strong condition in practice 
that seems unlikely to be satisfied 
(and difficult to check)
when learning complex distributions such as distributions of images. What would a reasonable weaker condition be, and in what sense can we still obtain reasonable samples?

\paragraph{Guarantees for learning the score function.} Our analysis assumes a $L^2$-estimate of the score function is given, but the question remains of when we can find such an estimate. What natural conditions on distributions allow their score functions to be learned by a neural network? Various works have considered the representability of data distributions by diffusion-like processes~\cite{tzen2019theoretical}, but the questions of optimization and generalization appear more challenging.

\subsection*{Acknowledgements}
We thank Andrej Risteski for helpful conversations. 
This work was done in part while HL was visiting the Simons Institute for the Theory of Computing. The work was supported in part by National Science Foundation via awards DMS-2012286 and CCF-1934964 (Duke Tripods).

\printbibliography

\newpage
\appendix

\section{Computations}


We start the proofs by collecting some preliminary results. 
In the following, we will consider the SDE 
\begin{align}
\label{e:sde-xt}
    dx_t &= f(x,t)\,dt + G(t)\,dw_t
\end{align}
and the interpolation of the discretization of an approximation
\begin{align}
    \label{e:sde-zt}
    dz_t &= \wh f(z_{t_-},t)\,dt + G(t)\,dw_t
\end{align}
when $t\ge t_-$. 
Let $P_t$ and $Q_t$ denote the law of $x_t$ and $z_t$, respectively. We will take $t_-=kh$ and $t\in [kh,(k+1)h)$. We will assume that $f,\wh f, G$ are continuous and the functions $f(\cdot, t)$, $\wh f(\cdot, t)$ are uniformly Lipschitz for each $t\in [kh, (k+1)h]$.

In this section, we will make some computations that will be used in both Sections~\ref{s:c} and~\ref{s:pc-proof}.
First, we derive how the density evolves in time.
\begin{lem}
\label{l:dqt}
Let $Q_t$ denote the law of the interpolated process~\eqref{e:sde-zt}.  
Then 
\begin{align*}
    \fc{\partial q_{t}(z)}{\partial t} & = \nb\cdot \ba{-q_{t}(z)\E\ba{\wh f(z_{t_-}, t)|z_t = z} + \fc{ G(t)G(t)^\top}{2}\nb q_t(z)}.
\end{align*}
\end{lem}
\begin{proof}
Let $q_{t|t_-}$ denote the distribution of $z_t$ conditioned on $z_{t_-}$. Then the Fokker-Planck equation gives
\begin{align*}
    \fc{\partial q_{t|t_-}(z|z_{t_-})}{\partial t} = -\nb q_{t|t_-}(z|z_{t_-}) \cdot [\wh f(z_{t_-}, t)] + \fc{G(t)G(t)^\top}{2}\Delta q_{t|t_-}(z|z_{t_-}) 
\end{align*}
Taking expectation with respect to $z_{t_-}$ we get
\begin{align*}
    \fc{\partial q_{t}(z)}{\partial t} & = \nb\cdot \int -q_{t|t_-}(z)\wh f(y, t)q_{t_-}(y)dy + \nb\cdot \ba{\fc{G(t)G(t)^\top}{2}\nb\int  q_{t|t_-}(z|y)q_{t_-}(y)dy}\\
    & = \nb\cdot q_t(z)\int \ba{-\wh f(y, t)q_{k|t}(y|z)dy + \fc{G(t)G(t)^\top}{2}\nb \int q_{t_-|t}(y|z)q_t(z)dy}.
\end{align*}
Note that for fixed $z$, $\int q_{t_-|t}(y|z)dy=1$. Hence
\begin{align*}
    \fc{\partial q_{t}(z)}{\partial t} & = \nb\cdot \ba{-q_{t}(z)\E\ba{\wh f(z_{t_-}, t)|z_t = z} + \fc{G(t)G(t)^\top}{2}\nb q_t(z)}.
\end{align*}
\end{proof}

We now use Lemma~\ref{l:dqt} to compute how the $\chi^2$-divergence between the approximate and exact densities changes. The following generalizes the calculation of~\cite{erdogdu2021convergence} in the case where $x_t$ is a non-stationary stochastic process. For simplicity of notation, from now on, we wil consider the case $G(t)$ being a scalar. 
\begin{lem}
\label{l:d-chi2}
Let $P_t$ and $Q_t$ be the laws of~\eqref{e:sde-xt} and~\eqref{e:sde-zt} for $G(t)=g(t)I_d$. Then
\[\pdd t\chi^2(q_t||p_t) = -g(t)^2\sE_{p_t}\pf{q_t}{p_t} + 2\E\ba{\an{\wh f(z_{t_-}, t) - f(z_t, t)  ,\nb\fc{q_t(z_t)}{p_t(z_t)}}}.\]
\end{lem}

\begin{proof}
The Fokker-Planck equation gives
\begin{align*}
    \fc{\partial p_t(x)}{\partial t} = \nb\cdot\ba{-f(x,t) p_t(x) + \fc{g(t)^2}{2}\nb p_t(x) }.
\end{align*}
We have
\begin{align*}
    \fc{d}{dt}\chi^2(q_t||p_t)
     = \fc{d}{dt}\int \fc{q_t(x)^2}{p_t(x)}dx = \int\ba{2\fc{\partial q_t(x)}{\partial t}\fc{q_t(x)}{p_t(x)} - \fc{\partial p_t(x)}{\partial t}\fc{q_t(x)^2}{p_t(x)^2}}dx.
\end{align*}
For the first term, by Lemma~\ref{l:dqt},
\begin{align}
\nonumber
    2\int \fc{\partial q_t(x)}{\partial t}\fc{q_t(x)}{p_t(x)} dx& = 2\int \nb\cdot \ba{-q_{t}(x)\E\ba{\wh f(z_0,t) |z_t = x} + \fc{g(t)^2}{2}\nb q_t(x)}\cdot \fc{q_t(x)}{p_t(x)}dx\\
    & = 2\int q_t(x) \an{\E\ba{\wh f(z_{0}, t)|z_t = x}, \nb \fc{q_t(x)}{p_t(x)}}dx-g(t)^2\int \an{\nb q_t(x), \nb \fc{q_t(x)}{p_t(x)}}dx.
    \label{e:comp1}
\end{align}
For the second term, using integration by parts,
\begin{align}
\nonumber
    -\int \fc{\partial p_t(x)}{\partial t}\fc{q_t(x)^2}{p_t(x)^2} dx & = \int \nb\cdot\ba{ f(x,t)p_t(x)- \fc{g(t)^2}{2}\nb p_t(x)}\cdot \fc{q_t(x)^2}{p_t(x)^2}dx\\
    \nonumber
    & = \int  -f(x,t)p_t(x)\nb \fc{q_t(x)^2}{p_t(x)^2} + \fc{g(t)^2}{2}\an{\nb p_t(x), \nb \fc{q_t(x)^2}{p_t(x)^2}} \dx\\
    \nonumber
    & = -2\int q_t(x)\an{f(x,t) , \nb\fc{q_t(x)}{p_t(x)}}dx\\
    &\ \ \ + g(t)^2\int \fc{q_t(x)}{p_t(x)} \an{\nb p_t(x), \nb \fc{q_t(x)}{p_t(x)}}dx .
    \label{e:comp2}
\end{align}
Note that
\begin{align*}
    \int \an{\nb q_t(x), \nb \fc{q_t(x)}{p_t(x)}} - \fc{q_t(x)}{p_t(x)}\an{\nb p_t(x), \nb \fc{q_t(x)}{p_t(x)}} &= 
    \int \an{\nb \fc{q_t(x)}{p_t(x)}, \nb \fc{q_t(x)}{p_t(x)}} q_t(x)\dx = \sE_{p_t}\pf{q_t}{p_t}.
\end{align*}
Combining~\eqref{e:comp1} and~\eqref{e:comp2},
\begin{align*}
    \fc{d}{dt}\chi^2(q_t||p_t)
    & = -g(t)^2\sE_{p_t}\pf{q_t}{p_t} + 2\int q_t(x)\an{\E\ba{\wh f(z_{t_-},t) - f(x, t)|z_t = x}, \nb\fc{q_t(x)}{p_t(x)}}dx\\
    & = -g(t)^2\sE_{p_t}\pf{q_t}{p_t} + 2\E\ba{\an{\wh f(z_{t_-}, t) - f(z_t, t)  ,\nb\fc{q_t(z_t)}{p_t(z_t)}}}.
\end{align*}
\end{proof}

Finally, we will make good use of the following lemma to bound the second term in Lemma~\ref{l:d-chi2}. 

\begin{lem}[cf. {\cite[Lemma 1]{erdogdu2021convergence}}]
\label{l:inp-young}
Let $\phi_t(x) = \fc{q_t(x)}{p_t(x)}$ and $\psi_t(x) = \phi_t(x)/\E_{p_t}\phi_t^2$. For any $c$ and any $\R^d$-valued random variable $u$, we have
\begin{align*}
\E\ba{\an{u ,\nb\fc{q_t(z_t)}{p_t(z_t)}}}
\le \E\ba{\ve{u}\ve{\nb\fc{q_t(z_t)}{p_t(z_t)}}}
   & \leq 
   C\cdot\E_{p_t}\phi_t^2\cdot  \E\ba{\ve{u}^2\psi_t(z_t)} + \fc{1}{4C} \sE_{p_t}\pf{q_t}{p_t}.
\end{align*}
\end{lem}
\begin{proof}
Note that $\E\psi_t(z_t)=1$ and the normalizing factor is $\E_{p_t}\phi_t^2 = \chi^2(q_t||p_t)+1$.
By Young's inequality,
\begin{align*}
    \E\ba{\an{u ,\nb\fc{q_t(z_t)}{p_t(z_t)}}}
    &= \E\ba{\an{
    u\sfc{q_t(z_t)}{p_t(z_t)}
    ,\sfc{p_t(z_t)}{q_t(z_t)} \nb\fc{q_t(z_t)}{p_t(z_t)}}}\\
    &\le C\E\ba{\ve{u}^2 \fc{q_t(z_t)}{p_t(z_t)}
    } + \rc{4C}\E_{p_t}\ba{\ve{\nb \fc{q_t(x)}{p_t(x)}}^2
    }\\
    & = C\E\ba{\ve{u}^2 \fc{q_t(z_t)}{p_t(z_t)}
    } + \fc{1}{4C} \sE_{p_t}\pf{q_t}{p_t}.
\end{align*}
\end{proof}

\section{Analysis for LMC}
\label{s:c}

Let $p$ be the probability density we wish to sample from. 
Suppose that we have an estimate $s$ of the score $\nb \ln p$. Our main theorem says that if the $L^2$ error $\E_p \ve{\nb \ln p - s}^2$ is small enough, then running LMC with $s$ for an \emph{appropriate} time results in a density that is close in \emph{TV distance} to a density that is close in \emph{$\chi^2$-divergence} to $p$. The following is a more precise version of Theorem~\ref{t:corrector-tv-chi2}.




\begin{thm}[LMC with $L^2$-accurate score estimate]\label{t:corrector-tv-chi2-precise}
Let $p:\R^d\to \R$ be a  probability density satisfying Assumption~\ref{a:p} with $\smf \ge 1$ and $s:\R^d\to \R^d$ be a score estimate  satisfying Assumption~\ref{a:score}(\ref{a:score-error}). 
Consider the accuracy requirement in $\TV$ and $\chi^2$:
$0<\ep_{\textup{TV}}<1$,
$0<\ep_{\chi}<1$, 
and suppose furthermore the starting distribution satisfies $\chi^2(p_0||p)\le K_\chi^2$. 
Then if
\[
\ep \le 
    \fc{\ep_{\TV}\ep_\chi^3}{174080\sqrt 5 d\smf^2\CLS^{5/2} (C_T\ln (2K_\chi/\ep_\chi^2)\vee 2K_\chi) },
\]
then running 
\eqref{e:lmc-se}
with score estimate $s$ and step size 
$
h=\fc{\ep_\chi^2}{2720d\smf^2 \CLS}
$
for any time $T\in [T_{\min},C_TT_{\min}]$,
where $T_{\min}=4\CLS \ln\pf{2K_\chi}{\ep_\chi^2}$, 
results in a distribution $p_T$ such that 
$p_T$ is $\ep_{\TV}$-far in TV distance from a distribution $\ol p_T$, 
where $\ol p_T$ satisfies
    $
        \chi^2(\ol p_T || p) \le \ep_\chi^2.
    $
In particular, taking $\ep_\chi=\ep_{\TV}$, we have the error guarantee that $\TV(p_T,p)=2\ep_{\TV}$.
\end{thm}

The main difficulty is that the stationary distribution of LMC using the score estimate may be arbitrarily far from $p$, even if the $L^2$ error of the score estimate is bounded. (See Section~\ref{a:far}.) 
Thus, a long-time convergence result does not hold, and an upper bound on $T$ is required, as in the theorem statement.

We instead proceed by showing that \emph{conditioned on not hitting a bad set}, if we run LMC using $s$, the $\chi^2$-divergence to the stationary distribution will decrease. This means that the closeness of the overall distribution (in TV distance, say) will decrease in the short term, despite it will increase in the long term, as the probability of hitting the bad set increases. This does not contradict the fact that the stationary distribution is different from $p$. By running for a moderate amount of time  (just enough for mixing), we can ensure that the probability of hitting the bad set is small, so that the resulting distribution is close to $p$. 
Note that we state the theorem with a $C_T$ parameter to allow a range of times that we can run LMC for.



More precisely, we prove Theorem~\ref{t:corrector-tv-chi2-precise} in two steps.

\paragraph{LMC under $L^\iy$ gradient error (Section~\ref{s:lmc-liy}, Theorem~\ref{t:lmc-iy}).} First, consider a simpler problem: proving a bound for $\chi^2$ divergence for LMC with score estimate $s$, when $\ve{s-\nb \ln p}$ is bounded everywhere, not just on average. For this, we follow the argument in~\cite{chewi2021analysis} for showing convergence of LMC in R\'enyi divergence; this also gives a bound in $\chi^2$-divergence. We define an interpolation of the discrete process and derive a upper bound for the derivative of R\'enyi divergence, $\pl_t\cR_q(q_t||p)$, using the log-Sobolev inequality for $p$. In the original proof, the error comes from the discretization error; here we have an additional error term coming from an inaccurate gradient, which is bounded by assumption. Note that a $L^2$ bound on $\nb f-s$ is insufficient to give an upper bound, as we need to bound $\E_{q_t\psi_t}[\|\nb f-s\|^2]$ for a different measure $q_t\psi_t$ that we do not have good control over. An $L^\iy$ bound works regardless of the measure.


\paragraph{Defining a bad set and bounding the hitting time (Section~\ref{s:hit}).} 
The idea is now to reduce to the case of $L^\iy$ error by defining the ``bad set" $B$ to be the set where $\ve{s-\nb f}\ge \ep_1$, where $\ep\ll \ep_1\ll 1$. This set has small measure by Chebyshev's inequality.
Away from the bad set, Theorem~\ref{t:lmc-iy} applies; it then suffices to bound the probability of hitting $B$. Technically, we define a coupling with a hypothetical process where the $L^\iy$ error is always bounded, and note that the processes disagree exactly when it hits $B$; this is the source of the TV error. 

We consider the probability of being in $B$ at times $0, h, 2h,\ldots$. we note that Theorem~\ref{t:corrector-tv-chi2-precise} bounds the $\chi^2$-divergence of this hypothetical process $X_t$ at time $t$ to $p$. If the distribution were actually $p$, then the probability $X_t\in B'$ is exactly $p(B')$; we expect the probability to be small even if the distribution is close to $p$. Indeed, by Cauchy-Schwarz, we can bound the probability $X\in B$ in terms of $P(B)$ and $\chi^2(q_t||p)$; this bound is given in Theorem~\ref{t:framework}.
Note that the eventual bound depends on $\chi^2(q_t||p)$, so we have to assume a warm start, that is, a reasonable bound on $\chi^2(q_0||p)$. 

\subsection{LMC under $L^\iy$ gradient error}
\label{s:lmc-liy}

The following gives a long-time convergence bound for LMC with inaccurate gradient, with error bounded in $L^\iy$; this may be of independent interest.

\lmcliy*
Following~\cite{chewi2021analysis}, convergence in R\'enyi divergence can also be derived; we only consider $\chi^2$-divergence because we will need a warm start in $\chi^2$-divergence for our application.
Note that by letting $N\to \iy$ and $h\to 0$, we obtain the following.
\begin{cor}
Keep the assumptions in Theorem~\ref{t:lmc-iy}. The stationary distribution $q$ of Langevin diffusion with score estimate $s$ satisfies
\[
\chi^2(q||p) \le 20\CLS \ep_1^2.
\]
\end{cor}

\begin{proof}[Proof of Theorem~\ref{t:lmc-iy}]
We follow the proof of~\cite[Theorem 4]{chewi2021analysis}, except that we work with the $\chi^2$ divergence directly, rather than the R\'enyi divergence, and have an extra term from the inaccurate gradient~\eqref{e:lmc-A3}.
Given $t\ge 0$, let $t_- = h\ff th$. 
Define the interpolated process by
\begin{align}
\label{e:interp2}
    dz_t &= s(z_{t_-}) \,dt + \sqrt 2\,dw_t,
\end{align}
and let $q_t$ denote the distribution of $X_t$ at time $t$, when $X_0\sim q_0$. 

By Lemma~\ref{l:d-chi2},
\begin{align}
\label{e:lmc-d-chi2}
    \pdd t\chi^2(q_t||p) = -2\sE_p\pf{q_t}{p} + 2\E\ba{\an{s(z_{t_-}) - \nb \ln p(z_t) ,\nb\fc{q_t(z_t)}{p(z_t)}}}.
\end{align}
By the proof of Theorem 4 in~\cite{chewi2021analysis}, 
\begin{align*}
    \ve{\nb \ln p(x_t) - \nb \ln p(x_{t_-})}^2 &\le 9 \smf^2(t-t_-)^2 \ve{\nb \ln p(x_t)}^2 + 6\smf^2 \ve{B_t - B_{t_-}}^2.
\end{align*}
Then 
\begin{align}
\nonumber
    \ve{s(z_{t_-}) - \nb \ln p(z_t)}^2 &\le  
    2\ve{\nb \ln p(z_{t_-}) - \nb \ln p(z_{t})}^2 + 
    2\ve{s(z_{t_-}) - \nb \ln p(z_{t_-})}^2 \\
    &\le 
    18 \smf^2(t-t_-)^2 \ve{\nb \ln p(z_t)}^2 + 12 \smf^2 \ve{B_t - B_{t_-}}^2 + 
    2\ep_1^2.
    \label{e:lmc-3terms}
\end{align}
Let $\phi_t :=q_t/p 
$ and $\psi_t := \fc{\phi_t}{\E_p (\phi_t^2)}$.
By Lemma~\ref{l:inp-young}, 
\begin{align}
\nonumber
    2\E\ba{\an{s(z_{t_-}) - \nb \ln p(z_t) ,\nb\fc{q_t(z_t)}{p(z_t)}}}
    &\le 2\E_{p}\phi_t^2\cdot  \E\ba{\ve{s(z_{t_-}) - \nb \ln p(z_t)}^2\psi_t(z_t)} + \fc{1}{2} \sE_p\pf{q_t}{p}\\
    &\le 
    A_1 + A_2 + A_3 + \rc 2\sE_{p}(\phi_t)
\label{e:lmc-As}
\end{align}
where $A_1,A_2,A_3$ 
are obtained by substituting in the 3 terms in~\eqref{e:lmc-3terms}, and given in~\eqref{e:lmc-A1}, \eqref{e:lmc-A2}, and \eqref{e:lmc-A3}. Let $V(x) = -\ln p(x)$. We consider each term in turn.
\begin{align}
\label{e:lmc-A1}
    A_1 :&= 36 \smf^2 (t-t_-)^2 \E_{p}\phi_t^2\cdot 
    \E\ba{\ve{\nb V(z_t)}^2\psi_t(z_t)}\\
    \nonumber
    &\le 
     36 \smf^2 (t-t_-)^2 \E_{p}\phi_t^2\cdot 
    \pa{
    \fc{4\sE_{p}(\phi_t)}{\E_p \phi_t^2} + 2d\smf
    } & \text{by~\cite[Lemma 16]{chewi2021analysis}}\\
    \nonumber 
    &\le \rc 2 \sE_p(\phi_t) + 72d\smf^3(t-t_-)^2(\chi^2(q_t||p)+1)
\end{align}
when $h^2\le \rc{288\smf^2}$. By~\cite[p. 15]{chewi2021analysis}
\begin{align}
\label{e:lmc-A2}
    A_2:&= 
    24\smf^2\E_p\phi_t^2 \cdot \E\ba{\ve{B_t-B_{t_-}}^2\psi_t(z_t)}\\
    \nonumber
    &\le 24\smf^2\E_p\phi_t^2 \cdot
    \pa{
        14d\smf^2(t-t_-) + 32h\CLS \fc{\sE_p(\phi_t)}{\E_p\phi_t^2}
    }\\
    \nonumber
    &\le 336d\smf^2(t-t_-)(\chi^2(q_t||p)+1)
    + \rc 2 \sE_p(\phi_t)
\end{align}
when $h\le \rc{1536\smf^2\CLS}$. Finally, 
\begin{align}
\label{e:lmc-A3}
    A_3:&=4\ep_1^2 \E_{p}\phi_t^2 = 4\ep_1^2(\chi^2(q_t||p)+1).
\end{align}
Combining~\eqref{e:lmc-d-chi2},~\eqref{e:lmc-As},~\eqref{e:lmc-A1},~\eqref{e:lmc-A2}, and~\eqref{e:lmc-A3} gives
\begin{align*}
    \pdd t\chi^2(q_t||p) &\le -\rc 2 \sE_p(\phi_t) + (\chi^2(q_t||p)+1) (72d\smf^3(t-t_-)^2 + 336 d\smf^2(t-t_-) + 4\ep_1^2)\\
    &\le -\rc{2\CLS} \chi^2(q_t||p) + 
     (\chi^2(q_t||p)+1) (72d\smf^3(t-t_-)^2 + 336 \smf^2d(t-t_-) + 4\ep_1^2)\\
    &\le -\rc{4\CLS} \chi^2(q_t||p) + (72d\smf^3(t-t_-)^2 + 336 d\smf^2(t-t_-) + 4\ep_1^2)
\end{align*}
if $h\le \prc{12\cdot 72d\smf^3\CLS}^{1/2}\wedge \rc{12\cdot 336d\CLS}$ and $\ep_1\le \prc{48\CLS}^{1/2}$.
Then for $t\in [kh,(k+1)h)$,
\begin{align*}
    \pdd t\pa{\chi^2(q_t||p) \exp\pf{t-t_-}{4\CLS}} &= \exp\pf{t-t_-}{4\CLS}(72d\smf^3(t-t_-)^2 + 336 d\smf^2(t-t_-) + 4\ep_1^2)\\
    &\le 73d\smf^3(t-t_-)^2 + 337 d\smf^2(t-t_-) + 5\ep_1^2.
\end{align*}
Integrating over $t\in [kh,(k+1)h)$ gives
\begin{align*}
    \chi^2(q_{(k+1)h}||p)
    &\le \exp\pa{-\fc{h}{4\CLS}}\chi^2(q_{kh}||p) + \fc{73}3d\smf^3h^3 + \fc{337}2 d\smf^2h^2 + 5\ep_1^2h\\
    &\le 
    \exp\pa{-\fc{h}{4\CLS}}\chi^2(q_{kh}||p) +
    170 d\smf^2h^2 + 5\ep_1^2h
\end{align*}
using $h\le \rc{12\sqrt 2\smf}$. Unfolding the recurrence and summing the geometric series gives
\begin{align*}
    \chi^2(q_{kh}||p) &\le \exp\pa{-\fc{kh}{4\CLS}}\chi^2(q_0||p) + 
    680 d\smf^2h\CLS + 20\ep_1^2 \CLS\\
    &\le \exp\pa{-\fc{kh}{4\CLS}}\chi^2(q_0||p) + 1
\end{align*}
when $h\le \rc{1360 d\smf^2\CLS}$ and $\ep_1^2\le \rc{40\CLS}$. We can check that the given condition on $h$ and the fact that $\smf \CLS\ge 1$ (Lemma~\ref{l:lc}) imply all the required inequalities on $h$.
\end{proof}

\subsection{Proof of Theorem~\ref{t:corrector-tv-chi2-precise}}
\label{ss:lmc-l2-pf}
\label{s:hit}
\begin{proof}[Proof of Theorem~\ref{t:corrector-tv-chi2-precise}]
We first define the bad set where the error in the score estimate is large,
\begin{align*}
B:&=\bc{\ve{\nb \ln p (x)-s(x)}>\ep_1}
\end{align*}
for some $\ep_1$ to be chosen. 

Given $t\ge 0$, let $t_- = h\ff th$. Given a bad set $B$, define the interpolated process by
\begin{align}
\label{e:interp}
    d\ol z_t &= b(\ol z_{t_-}) \,dt + \sqrt 2\,dw_t,\\
    \nonumber
    \text{where }
    b(z) &= \begin{cases}
    s(z),& z\nin B\\
    \nb \ln p(z), &z\in B
    \end{cases}.
\end{align}
In other words, run LMC using the score estimate as long as the point is in the
good set at the previous discretization step, and otherwise use the actual gradient $\nb \ln p$. Let $\ol q_t$ denote the distribution of $\ol z_t$ when $\ol z_0\sim q_0$; note that $q_{nh}$ is the distribution resulting from running LMC with estimate $b$ for $n$ steps and step size $h$. Note that this auxiliary process is defined only for purposes of analysis; it cannot be used for practical algorithm as we do not have access to $\nb f$.

We can couple this process with LMC using $s$ so that as long as $X_t$ does not hit $B$, the processes agree, thus satisfying condition~\ref{i:couple} of Theorem~\ref{t:framework}.

Then by Chebyshev's inequality,
\begin{align*}
    P(B) \le \pf{\ep}{\ep_1}^2=:\de.
\end{align*}
Let $T=Nh$.
Then by Theorem~\ref{t:lmc-iy},
\begin{align*}
\chi^2(\wt q_{kh}||p) &\le \exp\pa{-\fc{kh}{4\CLS}}\chi^2(q_0||p) + 
    680 d\smf^2h\CLS + 20\ep_1^2 \CLS\le \exp\pa{-\fc{kh}{4\CLS}}\chi^2(q_0||p) + 1.
\end{align*}
For this to be bounded by $\ep_\chi^2$, it suffices for the terms to be bounded by  $\fc{\ep_\chi^2}2, \fc{\ep_\chi^2}4, \fc{\ep_\chi^2}4$; this is implied by
\begin{align*}
    T &\ge 4\CLS \ln \pf{2K_\chi}{\ep_\chi^2}=:T_{\min}\\
    h &= \fc{\ep_\chi^2}{
    4392d\smf^2 \CLS}\\
    \ep_1 &= \fc{\ep_\chi}{4\sqrt {5\CLS}}.
\end{align*}
(We choose $h$ so that the condition in Theorem~\ref{t:lmc-iy} is satisfied; note $\ep_\chi\le 1$.)
By Theorem~\ref{t:framework}, \begin{align*}
\TV(q_{Nh},\ol q_{Nh}) 
&\le  
\sum_{k=0}^{N-1} (1+\chi^2(q_{kh}||p))^{1/2}P(B)^{1/2}\\
&\le 
\pa{\sumz k{N-1} \exp\pa{-\fc{kh}{8\CLS}}\chi^2(q_0||p)^{1/2} + 2} \de^{1/2}\\
&\le 
\pa{\pa{\sumz k{\iy} \exp\pa{-\fc{kh}{8\CLS}}K_\chi} + 2N}\fc{\ep}{\ep_1}\\
&\le 
\fc{\ep}{\ep_1}\pa{\fc{16\CLS}hK_\chi + 2N}. 
\end{align*}
In order for this to be $\le \ep_{\TV}$, it suffices for
\begin{align*}
    \ep &\le \ep_1\ep_{\TV} \pa{\rc{4N} \wedge \fc{h}{32\CLS K_\chi}}. 
\end{align*}
Supposing that we run for time $T$ where $T_{\min}\le T\le C_TT_{\min}$, we have that $N=\fc{T}h\le \fc{C_TT_{\min}}h$. Thus it suffices for 
\begin{align*}
    \ep &\le \ep_1\ep_{\TV} \pa{\fc{h}{4C_TT_{\min}}\wedge \fc{h}{32\CLS K_\chi}}\\
    &= \fc{\ep_\chi}{4\sqrt{5\CLS}} \cdot \ep_{\TV} \cdot 
    \fc{\ep_\chi^2}{2720d\smf^2\CLS} 
    \pa{\rc{16C_T\CLS \ln (2K_\chi/\ep_\chi^2)}\wedge \rc{32\CLS K_\chi} } \\
    &= 
    \fc{\ep_{\TV}\ep_\chi^3}{174080\sqrt 5 d\smf^2\CLS^{5/2} (C_T\ln (2K_\chi/\ep_\chi^2)\vee 2K_\chi) }. \qedhere
\end{align*}
\end{proof}

\subsection{Proof of Theorem~\ref{t:ald}}

We restate the theorem for convenience.
\aldtheorem*

\begin{proof}
We choose 
\begin{align*}
    h_M=\cdots=h_2 & = \Te\prc{dL^2\CLS} & 
    h_1 &= \Te\pf{dL^2\CLS}{\etv^2}\\
    T_{M-1}=\cdots = T_2&=\Te\pa{\CLS \ln \pf{M}{\etv}} & 
    T_1&=\Te\pa{\CLS \ln \pf{1}{\etv}},
\end{align*}
and $T_M=0$, $N_m= T_m/h$.

Choose the sequence $\si_{\min}^2=\si_1^2<\cdots<\si_M^2$ to be geometric with ratio $1+\Te\prc{\sqrt d}$.
Note that 
\begin{align*}
    \chi^2(N(0,\si_2^2I_d) || N(0,\si_1^2 I_d))
    &=
    \fc{\si_1^{d}}{\si_2^{2d}} (2\si_2^{-2} - \si_1^{-2})^{-d/2}-1
    = \pf{\si_2^2}{\si_1^2}^{-d/2} \pa{2-\pf{\si_2}{\si_1}^2}^{-\fc d2}. 
\end{align*}
For $\si_2^2 = (1+\ep)\si_1^2$, this equals $(1+\ep)^{-d/2} (1-\ep)^{-d/2} = (1-\ep^2)^{-d/2}-1$. For $\ep=\Te\prc{\sqrt d}$, this is $d\cdot O\prc{d} = O(1)$. Hence, the $\chi^2$-divergence between successive distributions $p_{\si_m^2}$ is $O(1)$. 
Choosing $\si_M^2=\Om(d(M_1+\CLS))$ ensures we have a warm start for the highest noise level by Lemma~\ref{l:warm}: $\chi^2(\ppr||p_{\si_M^2})=O(1)$. This uses $O\pa{\sqrt d\log \pf{d\CLS}{\si_{\min}^2}}$ noise levels.  

Write $p_m=p_{\si_m^2}$ for short. 
Let $q_m$ be the distribution of the final sample $x^{(m)}$. 
We show by downwards induction on $m$ that there is $\ol q_m$ such that 
\begin{align*}
    \TV(q_m, \ol q_m) &\le \fc{(M+1)-m}{M+1} \ep_{\TV} \\
    \chi^2(\ol q_m || p_m) &\le 
    {\pf{\etv}{4(M+1)}^2}.
\end{align*}
For $m=M$, this follows from the assumption on $\ep$ 
and Theorem~\ref{t:corrector-tv-chi2} with $K_\chi=O(1)$ (given by the warm start). 

Fix $m<M$ and suppose it holds for $m+1$. We use the closeness between $q_{m+1}$ and $p_{m+1}$ combined with $\chi^2(p_{m+1}||p_m)=O(1)$ to obtain 
compute how close $q_{m+1}$ and $p_m$ are.
Because the triangle inequality does not hold for $\chi^2$, we will incur an extra TV error. 

Let $\ol q_{m,m+1}$ be the distribution of the final sample if $x^{(m+1)}_0\sim \ol q_m$. We have $\TV(q_{m+1}, \ol q_{m,m+1}) \le \TV(q_m, \ol q_m) \le \fc{(M+1)-m}{M+1} \etv$.

By Markov's inequality, 
when $\chi^2(p_{m+1}||p_m)\le 1$,
\begin{align*}
    \Pj_{p_{m+1}} \pa{\fc{p_{m+1}}{p_m} \ge \fc{8(M+1)}{\etv}} &
    \le \fc{
    \chi^2(p_{m+1}||p_m)+1
    }{8(M+1)/\etv} 
    \le 
    \fc{\etv}{4(M+1)}.
\end{align*}
Let $\ol q_{m+1,m} = \one_{\bc{\fc{p_{m+1}}{p_m}\le \fc{8(M+1)}{\etv}}}\ol q_{m+1}\Big/{\int_{\bc{\fc{p_{m+1}}{p_m}\le \fc{8(M+1)}{\etv}}}\ol q_{m+1}}$.
Note that (using $\TV(\ol q_{m+1},p_{m+1})\le \sqrt{\chi^2(\ol q_{m+1}||p_{m+1})}\le\fc{\etv}{4(M+1)}$)
\begin{align}
\nonumber
    \Pj_{\ol q_{m+1}}\pa{\fc{p_{m+1}}{p_m} \ge \fc{8(M+1)}{\etv}}
    & \le \Pj_{p_{m+1}}\pa{\fc{p_{m+1}}{p_m} \ge \fc{8(M+1)}{\etv}} + \TV(\ol q_{m+1} , p_{m+1}) \\
    &\le \fc{\etv}{4(M+1)} + 
    {\fc{\etv}{4(M+1)}} \le \rc 2.
    \label{e:pratio-large}
\end{align}
so $\ol q_{m+1,m}\le 2\ol q_{m+1}$ and
\begin{align*}
    \chi^2(\ol q_{m+1,m}||p_m)+1
    &\le 
    2(\chi^2(\ol q_{m+1}||p_m)+1)
    \\
    &=\int_{\bc{\fc{p_{m+1}}{p_m}\le \fc{8(M+1)}{\etv}}} \fc{\ol q_{m+1}(x)^2}{p_{m+1}(x)^2}\cdot \fc{p_{m+1}(x)}{p_m(x)}p_{m+1}(x)\dx\\
    &\le \fc{8(M+1)}{\etv} (\chi^2(\ol q_{m+1}||p_{m+1})+1) \le 
    \fc{16(M+1)}{\etv}. 
\end{align*}
Let $\ol q_{m+1,m}'$ be the distribution of $x_{N_m}^{(m)}$ when $x_0^{(m)}\sim \ol q_{m+1,m}$. 
Then by assumption on $\ep$~\eqref{e:ep-almc} and  Theorem~\ref{t:corrector-tv-chi2} (with $K_\chi = 
{4\sfc{M+1}{\etv}}$, $\echi = \fc{\etv}{4(M+1)}$, and $\etv \leftarrow \fc{\etv}{2(M+1)}$), there is $\ol q_m$ such that $\TV(\ol q_{m,m+1}',\ol q_m)\le \fc{\etv}{2(M+1)}$ and $\chi^2(\ol q_{m}||p_m)\le 
{\fc{\etv}{4(M+1)}}$.
It remains to bound
\begin{align*}
    \TV(q_{m}, \ol q_{m})
    &\le 
    \TV(q_{m}, \ol q_{m,m+1}') + \TV(\ol q_{m,m+1}', \ol q_{m})\\
    &\le 
    \TV(q_{m+1}, \ol q_{m,m+1})
    +\fc{\etv}{2(M+1)}\\
    &\le \TV(q_{m+1}, \ol q_{m+1}) + \TV(\ol q_{m+1}, \ol q_{m+1,m})+ \fc{\etv}{2(M+1)}\\
    &\le \fc{(M+1)-(m+1)}{M+1}\etv + 
    {\Pj_{\ol q_{m+1}} \pa{\fc{p_{m+1}}{p_m} \ge \fc{8(M+1)}{\etv}}} + \fc{\etv}{2(M+1)}\\
    &\le \fc{(M+1)-(m+1)}{M+1}\etv + \fc{\etv}{2(M+1)} + \fc{\etv}{2(M+1)} 
    = 
    \fc{(M+1)-m}{M+1}\etv,
\end{align*}
where we use~\eqref{e:pratio-large} in the last line.
This finishes the induction step.

Finally, the theorem follows by taking $m=1$ and noting
\begin{align*}
    \TV(q_1, p_1)
    &\le 
    \TV(q_1,\ol q_1) + \TV(\ol q_1, p_1) \\
    &\le \TV(q_1,\ol q_1) + \sqrt{\chi^2(\ol q_1 || p_1)} \le \fc{M\etv}{M+1} + \fc{\etv}{4(M+1)}\le \etv. \qedhere
\end{align*}





\end{proof}

\section{Analysis for SGM based on reverse SDE's}
\label{s:pc-proof}

In this section, we analyze score-based generative models based on reverse SDE's. 
In Section~\ref{ss:p-1step}, we prove 
convergence of the predictor algorithm under $L^\iy$-accurate score estimate (Theorem~\ref{t:p-iy-simple}, restated as~\ref{t:p-iy}) using lemmas proved in Section~\ref{ss:diff-ineq},~\ref{ss:com}, \ref{ss:pert}, and \ref{ss:aux}.
In Section~\ref{ss:p-pf}, we prove convergence of the predictor algorithm under $L^2$-accurate score estimate
(Theorem~\ref{t:p}, restated as~\ref{t:p-precise}). In Section~\ref{ss:pc-pf}, we prove convergence of the predictor-corrector algorithm (Theorem~\ref{t:pc}).  


\subsection{Discretization and Score Estimation}\label{setting:C}
With a change of variable in~\eqref{general_reverse_sde}, we define the sampling process $x_t$ on $[0,T]$ by
\begin{align*}
    dx_t = [-f(x_t, T-t) + g(T-t)^2\nb\ln \Tilde{p}_{T-t}(x_t)]\,dt + g(T-t)\,dw_t,\ \ x_0\sim\Tilde{p}_T.
\end{align*}
Denoting the distribution of $x_t$ by $p_t$ and running the process from $0$ to $T$, we will exactly obtain  $p_T=\Tilde{p}_0$, which is the data distribution. In practice, we need to discretize this process and replace the score function $\nb\ln\Tilde{p}_{T-t}$ with the estimated score $\score$. With a general Euler-Maruyama method, we would obtain $\{z_k\}_{k=0}^N$ defined by
\begin{equation}\label{reverse_time_estimate_score}
z_{(k+1)h} = z_{kh} -
    h\cdot [f(z_{kh}, T-kh) - g(T-kh)^2\score (z_{kh}, T-kh)] + \sqrt{h}\cdot g(T-kh)\eta_{k+1},
\end{equation}
where $h = T/N$ is the step size and $\eta_k$ is a sequence of independent Gaussian random vectors. As we run~\eqref{reverse_time_estimate_score} from $0$ to $N$ with $h$ small enough, we should expect that the distribution of $z_T$ is close to that of $x_T$. However, in both SMLD or DDPM models, for fixed $z_k$, the integration 
\begin{align*}
    \int_{kh}^{(k+1)h}f(z_{kh}, T-t)\,dt\ \ \text{and}\ \   \score(z_{kh}, T-kh)\cdot\int_{kh}^{(k+1)h}g(T-t)^2\,dt
\end{align*}
can be exactly computed, as can the diffusion term. Therefore, we can consider the following process $z_t$ as an ``interpolation'' of~\eqref{reverse_time_estimate_score}:
\begin{align}
    dz_t = [-f(z_{kh}, T-t) + g(T-t)^2 \score(z_{kh}, T-kh)]\,dt + g(T-t)\,dw_t,\ \ t\in[kh, (k+1)h].\label{z_t}
\end{align}
Note that by running this process instead, we can reduce the discretization error. Now if we denote the distribution of $z_t$ by $q_t$, with $q_0\approx p_0$, we can expect that $q_T$ is close to $p_T$. Here the estimated score $\score$ satisfies for all $x$
\begin{equation}
\label{e:Mkh}
    \ve{\score(x, T-kh)-\nb\ln \Tilde{p}_{T-kh}(x)} \leq \Mkh, \quad k=0,1,\ldots,N.
\end{equation}
Observe that in either SMLD or DDPM, the function $g(t)^2$ is Lipschitz on $[0,T]$. So in the following sections, we will assume that $g(t)^2$ is $L_g$-Lipschitz on $[0,T]$.

\subsection{Predictor}
\label{ss:p-1step}
In this section, we present the main result (Theorem~\ref{t:p-iy}) on the one-step error of the predictor in $\chi^2$-divergence, which can be obtained by directly applying the Gronwall's inequality to the differential inequality derived in Lemma~\ref{d_chi2}. Note that Theorem~\ref{t:p-iy} is a more precise version of Theorem~\ref{t:p-iy-simple}; see the remark following the theorem.
\newcommand{\CtLeq}[0]{C_{t, L} &= \left\{
    \begin{array}{cc}
         32L^2\ \ \ \hfill\text{in SMLD},  \\
         (88C_t^2+400)L^2\ \ \ \text{in DDPM},
    \end{array}
\right.}
\newcommand{\tildeRteq}[0]{\Tilde{R}_{t} &= 9(C_t+1)}
\newcommand{\Rtkheq}[0]{R_d &= 300d+12} 
\newcommand{\CdLeq}[0]{C_{d, L} &= \left\{
    \begin{array}{cc}
         76L^2d\ \ \ \hfill\text{in SMLD},  \\
         6+94L^2d\ \ \ \text{in DDPM}
    \end{array}
\right. \le 100L^2d} 
\newcommand{\Rtkh}[0]{R_d}

\newcommand{\Etkh}[0]{E}
\newcommand{\Etkheq}[0]{\Etkh  &= 9(4\sms^2+1) + 8C_{d,L}}

\begin{thm}\label{t:p-iy}
With the setting in Section~\ref{setting:C}, assume $g$ is non-decreasing and let
\begin{align*}
    0 < h \leq \min_{kh\le t\le (k+1)h}
    \rc{g(T-kh)^2 (28L^2 + 10C_t + \E_{p_t}\ve{x}^2 + 64 C_{t,L} + 128 C_{d,L} + 360\sms^2 (\Tilde{R}_t + 2C_t \Rtkh))}
\end{align*}
where $C_t$ is the log-Sobolev constant of $p_t$, bounded in Lemma~\ref{l:lsi-noise}. Suppose that $\nb \ln p_t$ is $L$-Lipschitz for all $t\in [kh,(k+1)h]$, $s(\cdot, kh)$ is $\sms$-Lipschitz, $L,\sms\ge 1$, and $\Mkh$ is such that~\eqref{e:Mkh} holds.
Then 
\begin{align*}
    \chi^2(q_{(k+1)h}||p_{(k+1)h}) \leq \ba{\chi^2(q_{kh}||p_{kh}) + \int_{kh}^{(k+1)h}C_{t,kh}\,dt}e^{\int_{kh}^{(k+1)h}(-\fc{1}{8C_t} + 8\Mkh^2)g(T-t)^2\,dt}
\end{align*}
Here, 
\begin{align*}
    C_{t,kh} = \ba{8\Mkh^2 + \Etkh \cdot(t-kh)g(T-kh)^2}g(T-t)^2
\end{align*}
and 
\begin{align*}
\Etkheq\\
    \CtLeq \\
    \CdLeq \\
    \tildeRteq\\
    \Rtkheq 
\end{align*} 
are defined in~\eqref{constant:E_t_kh},~\eqref{constant:ctl},~\eqref{constant:cdl},~
\eqref{R_t} and~\eqref{R_t_kh}, respectively.
\end{thm}



\begin{proof}
The theorem follows from applying Gronwall's inequality to the result of Lemma~\ref{d_chi2}.
\end{proof}
\textbf{Remark.} 
Note that in \hyperref[i:DDPM]{DDPM}, $\Etkh =O(\sms^2 + \smf^2d)$. Therefore, when $g\equiv 1$,  $C_{t,kh} = O(\ep_1^2 + (\sms^2 + \smf^2d)h)$, where we denote the upper bound of $\Mkh$ for all $k\in\{0,...,N\}$ by $\ep_1$. 
Using the bound on the log-Sobolev constant  (Lemma~\ref{l:lsi-noise}) and second moment (Lemma~\ref{2ed_moment}) for DDPM, we note that the restriction on $h$ for all steps is implied by
\[
h = O\prc{\E_{\pdata}\ve{x}^2 + \CLS(\CLS+d)(\smf \vee \sms)^2}
\]
with appropriate constants.
Then we can conclude the first inequality in Theorem~\ref{t:p-iy-simple} by combining Theorem~\ref{t:p-iy} and Lemma~\ref{l:lsi-noise} and the second inequality from unfolding the first one and evaluating the geometric series.
Likewise, we have the following analogue for \hyperref[i:SMLD]{SMLD}, for which we omit the proof.


\begin{thm}[Predictor steps under $L^\iy$ bound on score estimate, SMLD]
\label{t:p-smld-simple}
Let $p:\R^d\to \R$ be a  probability density satisfying Assumption~\ref{a:p} 
and $s(\cdot, t):\R^d\to \R^d$ be a score estimate $s$ 
with error bounded in $L^\iy$ for each $t\in [0,T]$: 
    \[
    \ve{\nb \ln p - s(\cdot, t)}_\iy = 
    \max_{x\in \R^d} \ve{\nb \ln \wt p_t(x) - s(x, t)}]\le \ep_1.\]
Consider \hyperref[i:SMLD]{SMLD}. Let $C_T=\CLS+T$. Let $g\equiv 1$, $T\ge \CLS d$, and  $h=O\prc{\E_{p_0}\ve{x}^2 + C_Td(\smf\vee \sms)^2}$. Then
\begin{align*}
     \chi^2(q_{(k+1)h}||p_{(k+1)h}) &\leq
     \chi^2(q_{kh}||p_{kh}) e^{(-\rc{8C_{T-kh}}+8\ep_1^2)h} + O(\ep_1^2 h + (\sms^2 + \smf^2d)h^2)
\end{align*}
and letting $t=T-Nh$, if $\ep_1<\fc{1}{128C_T}$,
\begin{align*}
    \chi^2(q_{Nh}||p_{Nh}) &\leq
    \pf{\CLS+ t}{\CLS + T}^\fc{1}{16}
     \chi^2(q_{0}||p_{0})
     + O\pa{ \ln \pf{\CLS+T}{\CLS + t} \pa{\ep_1^2
     + (\sms^2+\smf^2d)h}}.
\end{align*}
Moreover, for $q_0=\ppr$, $q_0=\ph_{T}$, $\chi^2(q_0||p_0)\le \fc{\CLS d}{T}$.
\end{thm}

\noindent \textbf{Remark.} We note that in a sense \hyperref[i:SMLD]{SMLD} and \hyperref[i:DDPM]{DDPM} are equivalent, as we can get from one to the other by rescaling in time and space. First we recall that, as discussed in Section~\ref{s:pc}, all the SMLD models are equivalent under rescaling in time. Therefore we can assume $g(t) = e^{t/2}$ and consider the forward SDE for SMLD
\begin{align*}
    d x_t = e^{t/2} dw_t,
\end{align*}
where $w_t$ is a standard Brownian Motion. Now let $y_t = e^{-t/2}x_t$; then
\begin{align*}
    dy_t = -\fc12 y_t dt + dw_t,
\end{align*}
which is exactly \hyperref[i:DDPM]{DDPM} with $g(t) = 1$. Note that Theorem~\ref{t:p-smld-simple} uses a different parameterization for SMLD and the resulting complexity is slightly worse. 

\subsection{Differential Inequality}
\label{ss:diff-ineq}
Now we prove a differential inequality involving $\chi^2(q_t||p_t)$. As in \cite{chewi2021analysis}, the key difficulty is to bound the discretization error. We decompose it into two error terms and bound them in Lemma~\ref{l:error_A} and Lemma~\ref{l:error_B} separately.

In the following, we will let
\begin{align}
\label{e:Gkht}
    G_{kh,t}:&= \int_{kh}^t g(T-s)^2\,ds.
\end{align}

\begin{lem}\label{d_chi2}
Let $(q_t)_{0\leq t\leq T}$ denote the law of the interpolation~\eqref{z_t}. With the setting in Lemma~\ref{t:p-iy}, we have for $t\in[kh, (k+1)h]$,
\begin{align*}
    \fc{d}{dt}\chi^2(q_t||p_t)
    & \leq g(T-t)^2\ba{\pa{-\fc{1}{8C_t} + 8\Mkh^2}\chi^2(q_t||p_t) + \ba{8\Mkh^2 + \Etkh \cdot(t-kh)g(T-kh)^2}},
\end{align*}
where $C_t$ is the LSI constant of $p_t$, $\Mkh$ is the $L^\infty$-score estimation error at time $kh$ and $\Etkh $ is defined in~\eqref{constant:E_t_kh}.
\end{lem}
\begin{proof}
By Lemma~\ref{l:d-chi2} with
\begin{align*}
    \wh f(z_{kh},t) &\leftarrow -f(z_{kh}, T-t) + g(T-t)^2 s(z_{kh},T-kh)\\
    f(z,t) &\leftarrow -f(z,T-t) + g(T-t)^2 \nb \ln \wt p_{T-t}(z),
\end{align*}
we have
\begin{align*}
    \ddd t\chi^2(q_t||p_t)
    &= -g(T-t)^2 \sE_{p_t}\pf{q_t}{p_t} + 
    2\E\Big[\Big\langle 
    \pa{-f(z_{kh}, T-t) + g(T-t)^2 s(z_{kh}, T-kh)} \\
    &\quad - 
    \pa{-f(z,T-t) + g(T-t)^2 \nb \ln \wt p_{T-t}(z)}, \nb \fc{q_t}{p_t}\Big\rangle\Big]\\
    & = -g(T-t)^2\sE_{p_t}\pf{q_t}{p_t} + 2\E\ba{\an{f(z_t, T-t) - f(z_{kh}, T-t) ,\nb\fc{q_t(z_t)}{p_t(z_t)}}}\\
    &\ \ \ + 2g(T-t)^2\E\ba{\an{\score(z_{kh}, T-kh) -  \nb\ln\Tilde{p}_{T-t}(z_t), \nb\fc{q_t(z_t)}{p_t(z_t)}}}\\
    & =: -g(T-t)^2\sE_{p_t}\pf{q_t}{p_t} + A + B.
\end{align*}
By Lemma~\ref{l:error_A},
\begin{align*}
    A & \leq
        g(T-t)^2 \ba{2(\chi^2(q_t||p_t)+1) 
    \E\ba{\ve{z_t-z_{kh}}^2\psi_t(z_t)} + \rc{8} \sE_{p_t}\pf{q_t}{p_t}},
\end{align*}
while by Lemma~\ref{l:error_B},
\begin{align*}
    B 
    &\le \fc{1}{2}g(T-t)^2\sE_{p_t}\pf{q_t}{p_t} + 
    8g(T-t)^2\sms^2(\chi^2(q_t||p_t)+1)  \E\ba{\ve{z_t-z_{kh}}^2\psi_t(z_t)}\\
    &\quad + 8 \ba{\Mkh^2 + 
    \Gkht
    C_{d,L}}g(T-t)^2(\chi^2(q_t||p_t)+1).
\end{align*}
Therefore, for $h\leq \fc{1}{72g(T-kh)^2(4\sms^2+1)(\Tilde{R}_t\vee 2C_t \Rtkh)}\wedge \rc{128 g(T-kh)^2 C_{d,L}}$, using Lemma~\ref{zt_zkh_difference},
\begin{align*}
    \fc{d}{dt}\chi^2(q_t||p_t)
    & \leq
    -\fc38 g(T-t)^2 \sE_{p_t}\pf{q_t}{p_t}
    + 
    g(T-t)^2 (8\sms^2+2) (\chi^2(q_t||p_t)+1) 
    \E\ba{\ve{z_t-z_{kh}}^2\psi_t(z_t)}\\
    &\quad +8 \ba{\Mkh^2 + \Gkht C_{d,L}}g(T-t)^2(\chi^2(q_t||p_t)+1)\\
    &\le -\fc 38 g(T-t)^2\sE_{p_t}\pf{q_t}{p_t}\\
    &\quad + 
    9g(T-t)^2 (4L_s^2+1)
    \Gkht
    \ba{\Tilde{R}_{t}\sE_{p_t}\pf{q_t}{p_t} + R_{t, kh}(\chi^2(q_t||p_t)+1) }\\
    &\quad +8 \ba{\Mkh^2 + \Gkht  C_{d,L}}g(T-t)^2(\chi^2(q_t||p_t)+1)\\
    &\le -\fc 28 g(T-t)^2\sE_{p_t}\pf{q_t}{p_t}+ 
    g(T-t)^2 \rc{8C_t} \chi^2(q_t||p_t) +  8g(T-t)^2\Mkh^2\chi^2(q_t||p_t)
    \\
    &\quad + g(T-t)^2\ba{8\Mkh^2 + 8C_{d,L}\Gkht 
    + 9(4L_s^2+1)
    \Gkht \Rtkh}. 
\end{align*}
Using the fact that $p_t$ satisfies a log-Sobolev inequality with constant $C_t$,
\begin{align*}
    \ddd t\chi^2(q_t||p_t)
    &\le -\fc 2{8C_t} g(T-t)^2 \chi^2(q_t||p_t) + 
    \rc{8C_t} g(T-t)^2 \chi^2(q_t||p_t)+  8g(T-t)^2\Mkh^2\chi^2(q_t||p_t) \\
    &\quad + g(T-t)^2\ba{8\Mkh^2 + 8C_{d,L}\Gkht 
    + 9(4L_s^2+1)
    \Gkht \Rtkh}\\
    &\le  \pa{-\rc{8C_t} + 8\Mkh^2}g(T-t)^2 \chi^2(q_t||p_t) + 
    g(T-t)^2 [8\Mkh^2 + \Etkh (t-kh)g(T-kh)^2].
\end{align*}
where 
\begin{align}\label{constant:E_t_kh}
    \Etkheq.
\end{align}
\end{proof}
In order to bound the error terms $A$ and $B$, we will use Lemma~\ref{l:inp-young}.
Let $\phi_t(x) = \fc{q_t(x)}{p_t(x)}$ and $\psi_t(x) = \phi_t(x)/\E_{p_t}\phi_t^2$. Then $\E\psi_t(z_t)=1$ and in fact the normalizing factor $\E_{p_t}\phi_t^2 = \chi^2(q_t||p_t)+1$. 
We first deal with error term $A$.
\begin{lem}\label{l:error_A}
In the setting of Lemma~\ref{d_chi2}, we have the following bound for term $A$:
\begin{align*}
   &2\E\ba{\an{f(z_t, T-t) - f(z_{kh}, T-t) ,\nb\fc{q_t(z_t)}{p_t(z_t)}}}\\
    & \leq
    g(T-t)^2 \ba{2(\chi^2(q_t||p_t)+1) 
    \E\ba{\ve{z_t-z_{kh}}^2\psi_t(z_t)} + \rc{8} \sE_{p_t}\pf{q_t}{p_t}}.
\end{align*}
\end{lem}
\begin{proof}
In SMLD, $f(x,t)=0$ and hence $A=0$; while in DDPM, $f(x,t) 
= -\fc12 g(t)^2x$. Therefore, 
by Lemma~\ref{l:inp-young},
\begin{align*}
    &2\E\ba{\an{f(z_t, T-t) - f(z_{kh}, T-t) ,\nb\fc{q_t(z_t)}{p_t(z_t)}}}\\
    & = -g(T-t)^2\E\ba{\an{z_t - z_{kh} ,\nb\fc{q_t(z_t)}{p_t(z_t)}}}\\
    & \leq g(T-t)^2\ba{ 
    2\cdot\E_{p_t}\phi_t^2\cdot  \E\ba{\ve{z_t-z_{kh}}^2\psi_t(z_t)} + \fc{1}{
    8} \sE_{p_t}\pf{q_t}{p_t}}\\
    &= g(T-t)^2\ba{ 
    2(\chi^2(q_t||p_t)+1) \E\ba{\ve{z_t-z_{kh}}^2\psi_t(z_t)} + \fc{1}{
    8} \sE_{p_t}\pf{q_t}{p_t}}.
    \qedhere
\end{align*}
\end{proof}

Now we bound error term $B$.
\begin{lem}\label{l:error_B}
In the setting of Lemma~\ref{d_chi2}, we have the following bound for term $B$:
\begin{align*}
    B 
    &  \leq \fc{1}{2}g(T-t)^2\sE_{p_t}\pf{q_t}{p_t} + 
    8g(T-t)^2\sms^2(\chi^2(q_t||p_t)+1)  \E\ba{\ve{z_t-z_{kh}}^2\psi_t(z_t)}\\
    &\quad + 8 \ba{\Mkh^2 + \Gkht  C_{d,L}}g(T-t)^2(\chi^2(q_t||p_t)+1).
\end{align*}
\end{lem}
\begin{proof}
We first decompose the error:
\begin{align*}
     \E\ba{\an{\score(z_{kh}, T-kh) -  \nb\ln\Tilde{p}_{T-t}(z_t), \nb\fc{q_t(z_t)}{p_t(z_t)}}} 
    & = \E\ba{\an{\score(z_{kh}, T-kh) -   \score(z_t, T-kh), \nb\fc{q_t(z_t)}{p_t(z_t)}}}\\
    &\ \ \ + \E\ba{\an{ \score(z_t, T-kh) - \nb\ln p_{kh}(z_t), \nb\fc{q_t(z_t)}{p_t(z_t)}}}\\
    &\ \ \ + \E\ba{\an{ \nb\ln p_{kh}(z_t) - \nb\ln p_t(z_t), \nb\fc{q_t(z_t)}{p_t(z_t)}}}\\
    & =: B_1 + B_2 + B_3.
\end{align*}
Now we bound these error terms separately. For $B_1$, by the Lipschitz assumption, we have by Lemma~\ref{l:inp-young}, for a constant $C_2>0$ to be chosen later,
\begin{align*}
    B_1 & \leq \E\ba{\sms\ve{z_{kh}-z_t}\cdot\ve{\nb\fc{q_t(z_t)}{p_t(z_t)}}}\\
    & \leq  4\sms^2\cdot\E_{p_t}\phi_t^2\cdot  \E\ba{\ve{z_t-z_{kh}}^2\psi_t(z_t)} + \fc{1}{16} \sE_{p_t}\pf{q_t}{p_t}\\
    & = 4\sms^2(\chi^2(q_t||p_t)+1)  \E\ba{\ve{z_t-z_{kh}}^2\psi_t(z_t)} + \fc{1}{16} \sE_{p_t}\pf{q_t}{p_t}.
\end{align*}
For $B_2$, recalling the assumption that $\ve{\score(x, T-kh)-\nb\ln p_{kh}(x)}\leq \Mkh$ for all $x$, we have 
by Lemma~\ref{l:inp-young}
\begin{align}
\nonumber
    B_2 & \leq 4\E\ba{\ve{\score(z_t, T-kh)-\nb\ln p_{kh}(z_t)}^2\psi_t(z_t)} \cdot \E_{p_t}[\phi_t^2] + \fc{1}{16} \sE_{p_t}\pf{q_t}{p_t}\\
    & \leq 4\Mkh^2(\chi^2(q_t||p_t) + 1) + \fc{1}{16} \sE_{p_t}\pf{q_t}{p_t}.
    \label{e:B2}
\end{align}
Now for the last error term $B_3$, we have by Lemma~\ref{l:inp-young} that
\begin{align}
\nonumber
    B_3 
    & \leq 4\E_{p_t}\phi_t^2\cdot \E\ba{\ve{\nb\ln p_{kh}(z_t) - \nb\ln p_t(z_t)}^2\psi_t(z_t)} + \fc{1}{16} \sE_{p_t}\pf{q_t}{p_t}\\
    & \leq 4K_{t,kh}(\chi^2(q_t||p_t) + 1) + \fc{1}{16} \sE_{p_t}\pf{q_t}{p_t}.
    \label{e:B3}
\end{align}
Here $K_{t,kh}$ is the bound for $\E\ba{\psi_t(z_t)\ve{\nb\ln p_{kh}(z_t) - \nb\ln p_{t}(z_t)}^2}$ obtained in Lemma~\ref{l:perturb_error}:
\begin{align*}
    K_{t,kh} := \Gkht \ba{\fc{C_{t,L}}{\chi^2(q_t||p_t) + 1}\cdot\sE_{p_t}\pf{q_t}{p_t} + C_{d,L}}
\end{align*}
where $C_{t,L}$ and $C_{d, L}$ are constants defined in~\eqref{constant:ctl} and~\eqref{constant:cdl} respectively. Hence
\begin{align*}
    B_3 & \leq 4\Gkht \ba{C_{t,L}\sE_{p_t}\pf{q_t}{p_t} + C_{d,L}(\chi^2(q_t||p_t)+1)}+ \fc{1}{16}\sE_{p_t}\pf{q_t}{p_t}.
\end{align*}
Combining all these results, we finally obtain the bound for error term $B$ in Lemma~\ref{d_chi2}: for $h\leq\fc{1}{64 C_{t,L} g(T-kh)^2}$,
\begin{align*}
    B & = 2g(T-t)^2 (B_1+B_2+B_3)\\
    & \leq \fc{3}{8}g(T-t)^2\sE_{p_t}\pf{q_t}{p_t} + 
    8g(T-t)^2\sms^2(\chi^2(q_t||p_t)+1)  \E\ba{\ve{z_t-z_{kh}}^2\psi_t(z_t)}\\
    &\quad + 8 C_{t,L} g(T-t)^2\Gkht  \sE_{p_t}\pf{q_t}{p_t}\\
    &\quad + 8 \ba{\Mkh^2 + \Gkht  C_{d,L}}g(T-t)^2(\chi^2(q_t||p_t)+1)\\
    &  \leq \fc{1}{2}g(T-t)^2\sE_{p_t}\pf{q_t}{p_t} + 
    8g(T-t)^2\sms^2(\chi^2(q_t||p_t)+1)  \E\ba{\ve{z_t-z_{kh}}^2\psi_t(z_t)}\\
    &\quad + 8 \ba{\Mkh^2 + \Gkht  C_{d,L}}g(T-t)^2(\chi^2(q_t||p_t)+1). \qedhere
\end{align*}
\end{proof}

\subsection{Change of Measure}
\label{ss:com}
As shown in Lemma~\ref{l:error_A} and Lemma~\ref{l:error_B}, the key to the proof of Lemma~\ref{d_chi2} is bounding the discretization error $A$ and $B$. The difficulty is that these errors usually have the form of $\E_{\psi_t q_t}\ba{\ve{u(x)}^2}$ for some function $u:\mathbb R^d\rightarrow \mathbb R^d$, while it is usually easier to bound those expectations over the original probability measure or our target distribution $p_t$. Therefore, as discussed in \cite[Section 5.1]{chewi2021analysis}, our task is to bound these error terms under a complicated change of measure. We first state such a result with respect to the gradient of the potential.
\begin{lem}~\cite[Lemma 16]{chewi2021analysis}\label{l:trivial_bound}
Assume that $p(x)\propto e^{-V(x)}$ is a density in $\mathbb R^d$ and $\nb V(x)$ is $L$-Lipschitz. Then for any probability density $q$, it holds that
\[
\E_q\ba{\ve{\nb V}^2}\leq 4\E_p\ba{\ve{\nb \sqrt{\fc{q(x)}{p(x)}}}^2} + 2dL =\E_q\ba{\ve{\nb \ln{\fc{q(x)}{p(x)}}}^2} +2dL.
\]
\end{lem}
\begin{proof}
Define the Langevin diffusion w.r.t. $p(x)$:
\begin{align*}
    d x_t = -\nb V(x_t) \,dt + \sqrt{2}\,dw_t,
\end{align*}
where $B_t$ is a standard Brownian Motion in $\mathbb R^d$. Let $\mathcal{L}$ be the corresponding infinitesimal generator, i.e., $\mathcal{L}f =\an{\nb V, \nb f} - \Delta f$. Observe that $\mathcal{L}V = \ve{\nb V}^2 - \Delta V$ and $\E_p\mathcal{L}f = 0$ for any $f$, so
\begin{align*}
    \E_q\ba{\ve{\nb V}^2} & = \E_q\mathcal{L}V +  \E_q \Delta V \\
    &\leq \int \mathcal{L}V\pa{\fc{q(x)}{p(x)} -1}p(x)dx + dL = \int\an{\nb V,\nb\fc{q(x)}{p(x)}}p(x)dx + dL\\
    & = 2\int\an{\sqrt{\fc{q(x)}{p(x)}}\nb V, \nb \sqrt{\fc{q(x)}{p(x)}}}p(x)dx + dL \\
    &\leq \fc12 \E_q\ba{\ve{\nb V}^2} + 2\E_p\ba{\ve{\nb \sqrt{\fc{q(x)}{p(x)}}}^2} + dL.
\end{align*}
Rearrange this inequality to obtain the desired result.
\end{proof}
Now applying this Lemma to $p = p_t$ and $q=\psi_t q_t$, we get immediately the following corollary. Note that $\psi_tq_t$ is a density function because $\int \psi_t(x)q_t(x)\dx = \int \fc{q_t(x)}{p_t(x)}q_t(x)\dx/\E_{p_t}\phi_t^2 = 1$ and $\psi_t(x) q_t(x)\geq 0$ for any $x\in \R^d$.
\begin{cor}\label{l:DV_nb_ln_p}
In the setting of Lemma~\ref{d_chi2}, it holds that
\begin{align*}
    \E\ba{\psi_t(z_t)\ve{\nb \ln p_t(z_t)}^2} &\leq
    \fc{4}{\chi^2(q_t||p_t)+1} \cdot\sE_{p_t}\pf{q_t}{p_t}+2dL.
\end{align*}
\end{cor}
\begin{proof}
Applying Lemma~\ref{l:trivial_bound} to the density $\psi_t q_t$ yields
\begin{align*}
    \E_{\psi_t q_t}\ba{\ve{\nb\ln p_t(x)}^2}\leq \E_{\psi_t q_t}\ba{\ve{\nb\ln\fc{\psi_t(x)q_t(x)}{p_t(x)}}^2} + 2dL = \fc{4}{\chi^2(q_t||p_t)+1} \cdot\sE_{p_t}\pf{q_t}{p_t}+2dL.
\end{align*}
\end{proof}
Note that we cannot expect analogous results for a general $u(x)$ as in Lemma~\ref{l:trivial_bound}. In the general case, we apply the Donsker-Varadhan variational principle, which states that for probability measures $p$ and $q$,
\begin{align*}
    \E_q \ve{u(x)}^2 \leq \KL(q||p) + \ln\E_p\exp{\ve{u(x)}^2 }.
\end{align*}
Towards this end, we first need to analyze $\KL(\psi_t q_t ||p_t)$.
\begin{lem}\label{KL_psi_q&p}
Let $\phi_t(x) = \fc{q_t(x)}{p_t(x)}$ and $\psi_t(x) = \phi_t(x)/\E_{p_t}\phi_t^2$. 
If $p_t$ satisfies a LSI with constant $C_t$, then
\begin{align*}
    \KL(\psi_t q_t||p_t) & \leq \fc{2C_t}{\chi^2(q_t||p_t)+1}\cdot\sE_{p_t}\pf{q_t}{p_t}.
\end{align*}
\end{lem}
\begin{proof}
Since $p_t$ satisfies LSI with constant $C_t$,
\begin{align*}
    \KL(\psi_t q_t||p_t) & \leq \fc{C_t}{2}\int \ve{\nb\ln\fc{\psi_t(x) q_t(x)}{p_t(x)}}^2\psi_t(x)q_t(x)dx\\
    & = 2C_t\int \ve{\nb \ln\fc{q_t(x)}{p_t(x)}}^2\psi_t(x)q_t(x)dx\\
    & = 2C_t\int \ve{\nb\fc{q_t(x)}{p_t(x)}}^2\fc{\psi_t(x)p_t(x)^2}{q_t(x)}dx\\
    & = \fc{2C_t}{\chi^2(q_t||p_t)+1}\cdot \int\ve{\nb\fc{q_t(x)}{p_t(x)}}^2 p_t(x)dx\\
    & = \fc{2C_t}{\chi^2(q_t||p_t)+1}\cdot\sE_{p_t}\pf{q_t}{p_t}.\qedhere
\end{align*}
\end{proof}
With this in hand, we are ready to bound the second moment of $\psi_t q_t$ as well as the variance of a Gaussian random vector with respect to this measure:
\begin{lem}\label{second_moment}
With the setting of Lemma~\ref{d_chi2}, we have
\begin{align*}
    \E\ba{\psi_t(z_t)\ve{z_t}^2} & \leq  \fc{2C_t^2}{\chi^2(q_t||p_t)+1}\cdot\sE_{p_t}\pf{q_t}{p_t} + \fc12 \E_{p_t}\ba{\ve{x}^2} + \fc12 C_t,
\end{align*}
where $C_t$ is the LSI constant of $p_t$, which is bounded in Lemma~\ref{l:lsi-conv}, and the second moment of $p_t$ is bounded in Lemma~\ref{2ed_moment}.
\end{lem}
\begin{proof}
Since $p_t$ has LSI constant $C_t$, by Donsker-Varadhan variational principle,
\begin{align*}
    \E\ba{\psi_t(z_t)\ve{z_t}^2} & = \fc{2}{s}\E_{\psi_t q_t}\ba{\fc{s}{2}\ve{x}^2}\leq \fc{2}{s}\ba{\KL(\psi_t q_t||p_t) + \ln\E_{p_t}\ba{e^{\fc{s}{2}\ve{x}^2}}}
\end{align*}
for any $s>0$. By Lemma~\ref{l:herbst}, for any $s\in[0, \fc{1}{C_t})$, we have
\begin{align*}
    \E_{p_t}\ba{e^{\fc{s}{2}\ve{x}^2}} \leq \fc{1}{\sqrt{1-C_t\cdot s}}\exp{\ba{\fc{s}{2(1-C_t\cdot s)}(\E_{p_t}\ve{x})^2}}.
\end{align*}
Now choose $s=\fc{1}{2C_t}$, we have
\begin{align*}
   \E_{p_t}\ba{e^{\fc{s}{2}\ve{x}^2}}\leq \sqrt{2}\exp\ba{\fc{1}{2C_t}(\E_{p_t}\ve{x})^2}.
\end{align*}
Hence
\begin{align*}
    \E\ba{\psi_t(z_t)\ve{z_t}^2} & \leq C_t\cdot\ba{\KL(\psi_t q_t||p_t) + \fc{1}{2C_t}\E_{p_t}\ba{\ve{x}^2} + \fc{\ln 2}{2}}.
\end{align*}
Now with the bound of $\KL(\psi_t q_t||p_t)$ in Lemma~\ref{KL_psi_q&p}, we obtain
\begin{align*}
   \E\ba{\psi_t(z_t)\ve{z_t}^2} & \leq  \fc{2C_t^2}{\chi^2(q_t||p_t)+1}\cdot\sE_{p_t}\pf{q_t}{p_t} + \fc12 \E_{p_t}\ba{\ve{x}^2} + \fc12 C_t. \qedhere
\end{align*}
\end{proof}

\begin{lem}\label{second_moment_diffusion}
With the setting of Lemma~\ref{d_chi2},
\begin{align*}
    \E\ba{\psi_t(z_t)\ve{\int_{kh}^t g(T-s)dw_s}^2} 
    & \leq 2
    \int_{kh}^t g(T-s)^2\,ds
    \cdot\ba{\fc{8C_t}{\chi^2(q_t||p_t)+1}\cdot\sE_{p_t}\pf{q_t}{p_t} + d + 8\ln 2},
\end{align*}
where $C_t$ is the LSI constant of $p_t$.
\end{lem}
\begin{proof}
Note that $\int_{kh}^t g(T-s)dw_s$ is a Gaussian random vector with variance $\int_{kh}^t g(T-s)^2 ds\cdot I_d$. Using the Donsker-Varadhan variational principle, for any random variable $X$,
\begin{align*}
    \Tilde{\mathbb E} X \leq \KL(\Tilde{\mathbb P}||\mathbb P) + \ln \mathbb E \exp X.
\end{align*}
Applying this to $X=c\bigg(\ve{\int_{kh}^t g(T-s)dw_s} - \E \ve{\int_{kh}^t g(T-s)dw_s}\bigg)^2$ for a constant $c>0$ to be chosen later, we can bound
\begin{align*}
    \Tilde{\mathbb E}\ve{\int_{kh}^t g(T-s)dw_s}^2 & \leq 2\E\ba{\ve{\int_{kh}^t g(T-s)dw_s}^2} \\
    &\ \ \ + \fc2c\ba{\KL(\Tilde{\mathbb P}||\mathbb P) + \ln \mathbb E \exp\bigg(c\Big(\ve{\int_{kh}^t g(T-s)dw_s} - \E \ve{\int_{kh}^t g(T-s)dw_s}\Big)^2\bigg)},
\end{align*}
where $\fc{d\Tilde{\mathbb P}}{d \mathbb P} = \psi_t(z_t)$. Now following \cite[Theorem 4]{chewi2021analysis}, we set $c = \fc{1}{8\int_{kh}^t g(s)^2 ds}$, so that 
\begin{align*}
    \E \exp\ba{\fc{\pa{\ve{\int_{kh}^t g(T-s)dw_s} - \E \ve{\int_{kh}^t g(T-s)dw_s}}^2}{8\int_{kh}^t g(s)^2 ds}}\leq 2.
\end{align*}
Next, using the LSI for $p_t$, we have
\begin{align*}
    \KL(\Tilde{\mathbb P}||\mathbb P) & = \E_{\psi_t q_t}\ln \psi_t = \E_{\psi_t q_t}\ln \fc{\phi_t}{\E_{p_t}\phi_t^2} = \fc12\E_{\psi_t q_t}\ln \fc{\phi_t^2}{(\E_{p_t}\phi_t^2)^2}\\
    &= \fc12 \ba{\E_{\psi_t q_t}\ln \fc{\phi_t^2}{\E_{p_t}\phi_t^2} - \ln\E_{p_t}\phi_t^2}
    = \rc 2\ba{\E_{\psi_t q_t}\ln \fc{\psi_tq_t}{p_t} - \ln\E_{p_t}\phi_t^2}.
\end{align*}
Noting that $\E_{p_t}\phi_t^2 = \chi^2(q_t||p_t) + 1 \geq 1$, we have that
\begin{align*}
    \KL (\Tilde{\mathbb P}||\mathbb P) & \leq \fc12 \KL (\psi_t q_t||p_t)\leq \fc{C_t}{\chi^2(q_t||p_t)+1}\cdot\sE_{p_t}\pf{q_t}{p_t},
\end{align*}
where the last inequality is due to Lemma~\ref{KL_psi_q&p}. We have proved
\begin{align}
    &\E\ba{\psi_t(z_t)\ve{\int_{kh}^t g(T-s)\,dw_s}^2}\notag \\
    & \leq 2d\int_{kh}^t g(T-s)^2\,ds + 16\int_{kh}^t g(T-s)^2 ds\cdot\ba{\fc{C_t}{\chi^2(q_t||p_t)+1}\cdot\sE_{p_t}\pf{q_t}{p_t} + \ln 2}\notag\\
    & \leq 2
    \int_{kh}^t g(T-s)^2\,ds
    \cdot\ba{\fc{8C_t}{\chi^2(q_t||p_t)+1}\cdot\sE_{p_t}\pf{q_t}{p_t} + d + 8\ln 2}. \qedhere
\end{align}
\end{proof}

\subsection{Perturbation Error}
\label{ss:pert}
In the previous section, we bound errors in the form of $\E_{\psi_t q_t}\ve{u(x)}^2$ with a change of measure technique, where $\ve{u(x)}^2$ is easy to bound with respect to the original measure or $p_t$. However, this is not always the case for the errors we are considering. In this section, we aim to bound $\E_{\psi_t q_t}\ba{\ve{\nb\ln p_{kh}(x) - \nb\ln p_{t}(x)}^2}$, where, as discussed in Lemma~\ref{l:perturb_error}, $p_{kh}$ can be regarded as a perturbed version of $p_t$ with some Gaussian noise. We first provide a point-wise bound for SMLD (Lemma~\ref{l:perturb_SMLD}) and DDMP (Lemma~\ref{l:perturb_DDPM}), respectively and then use them to bound the expectation with respect to $\psi_t q_t$.

\begin{lem}\label{l:perturb_SMLD}
Suppose that $p(x) \propto e^{-V(x)}$ is a probability density on $\R^d$, where $V(x)$ is $L$-smooth, and let $\ph_{\sigma^2}(x)$ be the density function of $N(0,\si^2I_d)$.  
Then for $L\le \rc{2\si^2}$,
\[
\ve{\nb \ln \fc{p(x)}{(p*\ph_{\si^2})(x)}} \le 6L\si d^{1/2} + 2L\si^2 \ve{\nb V(x)}.
\]
\end{lem}
\begin{proof}
Note that 
\begin{align*}
    \nb \ln p*\ph_{\si^2}(x) = 
   \fc{\int_{\R^d} - \nb V(y) e^{-V(y)}e^{-\fc{\ve{x-y}^2}{2\si^2}}\dy}{\int_{\R^d} e^{-V(y)}e^{-\fc{\ve{x-y}^2}{2\si^2}}\dy}
   = -\E_{p_{x, \si^2}}\nb V(y),
\end{align*}
where $p_{x, \si^2}$ denotes the probability density
\[
p_{x, \si^2}(y) \propto p(y) e^{-\fc{\ve{y-x}^2}{2\si^2}} 
\]
so when $V$ is $L$-smooth,
\begin{align*}
    \ve{\nb \ln \fc{p(x)}{p*\ph_{\si^2}(x)}} & = \ve{\E_{p_{x, \si^2}}[\nb V(y)-\nb V(x)]}\\
    &\le \E_{p_{x, \si^2}}[L\ve{y-x}]
\end{align*}
We now write
\begin{align*}
    \E_{p_{x, \si^2}}\ve{y-x}
    &\le 
    \E_{p_{x, \si^2}} \ve{y - \E_{p_{x, \si^2}}y} 
    + \ve{\E_{p_{x, \si^2}}y - y^*}
    + \ve{y^*-x},
\end{align*}
where $y^*\in \amax_{y}p_{x,\si^2}(y)$ is a mode of the distribution $p_{x,\si^2}$. 
We now bound each of these terms.
\begin{enumerate}
    \item For the first term, note  that $p_{x,\si^2}$ is $\pa{\rc{\si^2}-L}$-strongly convex, so 
    satisfies a Poincar\'e inequality with constant $\pa{\rc{\si^2}-L}^{-1}$. Thus 
    \begin{align*}
        \E_{p_{x, \si^2}}\ve{y-x}
    &\le \E_{p_{x, \si^2}}[\|y-\E_{p_{x, \si^2}} y\|^2]^{1/2}\\
    & = \pa{\sumo id \Var_{p_{x,\si^2}}(y_i)}^{1/2}
    \le \pa{d\pa{\rc{\si^2}-L}^{-1}}^{1/2}.
    \end{align*}
    \item For the second term, by 
    Lemma~\ref{l:mean-mode}, 
    noting that $V(y) + \fc{\ve{x-y}^2}{2\si^2}$ is $\pa{\rc{\si^2}+L}$-smooth,
    \begin{align*}
        \ve{\E_{p_{x, \si^2}}y - y^*} &\le \pa{\rc{\si^2}-L}^{-1/2} d^{1/2}\pa{5+\ln \pa{\pa{\rc{\si^2}-L}^{-1}\pa{\rc{\si^2}+L}}}^{1/2}\\
        &\le \pa{\rc{\si^2}-L}^{-1/2}d^{1/2} \pa{5+\ln \fc{1+L\si^2}{1-L\si^2}}^{1/2}\\
        &\le\sqrt 7 \pa{\rc{\si^2}-L}^{-1/2}d^{1/2},
    \end{align*}
    where the last inequality uses $\si^2\le \rc{2L}$. 
    \item For the third term, we note that the mode satisfies
    \begin{align*}
        \nb V(y^*) + \fc{y^*-x}{\si^2} & = 0\\
        -\fc{y^*-x}{\si^2} &= \nb V(y^*) = (\nb V(y^*)-\nb V(x)) + \nb V(x)\\
        \rc{\si^2}\ve{y^*-x} &\le \ve{\nb V(x)}+L\ve{y^*-x}\\
        \ve{y^*-x} &\le \pa{\rc{\si^2}-L}^{-1} \ve{\nb V(x)}.
    \end{align*}
\end{enumerate}
Putting these together and using $\pa{\rc{\si^2}-L}^{-1}\le 2$, we obtain
\begin{align*}
    \ve{\nb \ln \fc{p(x)}{p*\ph_{\si^2}(x)}}
    &\le (\sqrt 7+1)L\pa{\rc{\si^2}-L}^{-1/2}d^{1/2} + L \pa{\rc{\si^2}-L}^{-1} \ve{\nb V(x)}\\
    &\le 
    6L\si d^{1/2} + 2L\si^2 \ve{\nb V(x)}. \qedhere
\end{align*} 
\end{proof}

\begin{lem}\label{l:perturb_DDPM}
With the setting in Lemma~\ref{l:perturb_SMLD} and the notation $p_\alpha(x) = \alpha^d p(\alpha x)$ for $\alpha\geq 1$, we have that for $L\leq \fc{1}{2\al^2\sigma^2}$,
\[
\ve{\nb \ln \fc{p(x)}{(p_\alpha *\ph_{\si^2})(x)}} \le 6\alpha^2 L\sigma d^{1/2} + (\alpha+2\alpha^3 L\sigma^2)(\alpha -1) L\ve{x} + (\alpha-1 + 2\alpha^3 L\sigma^2)\ve{\nb V(x)}.
\]
\end{lem}
\begin{proof}
Note $p_\alpha(x)$ is also a probability density in $\mathbb R^d$. By the triangle inequality,
\[
\ve{\nb \ln \fc{p(x)}{(p_\alpha *\ph_{\si^2})(x)}} \leq \ve{\nb \ln \fc{p(x)}{p_\alpha (x)}} + \ve{\nb \ln \fc{p_\alpha (x)}{(p_\alpha *\ph_{\si^2})(x)}}.
\]
Without loss of generality, we can assume that $p(x) = e^{-V(x)}$; then $p_\alpha(x) = \alpha^d e^{-V(\alpha x)}$. Hence
\begin{align*}
   \ve{\nb \ln \fc{p(x)}{p_\alpha (x)}} & = \ve{\alpha\nb V(\alpha x) - \nb V(x)} \\
   & \leq \ve{\alpha\nb V(\alpha x) - \alpha \nb V(x)} + \ve{\alpha \nb V(x) -  \nb V(x)}\\
   & \leq \alpha(\alpha-1) L\ve{x} + (\alpha-1)\ve{\nb V(x)}.
\end{align*}
Since $\alpha\nb V(\alpha x)$ is $\alpha^2L$-Lipschitz, by Lemma~\ref{l:perturb_SMLD}, 
\begin{align*}
    \ve{\nb \ln \fc{p_\alpha (x)}{(p_\alpha *\ph_{\si^2})(x)}}\leq 6\alpha^2 L\sigma d^{1/2} + 2\alpha^3L\sigma^2\ve{\nb V(\alpha x)}.
\end{align*}
By the Lipschitz assumption,
\begin{align*}
    \ve{\nb V(\alpha x)} & \leq \ve{\nb V(\alpha x) - \nb V(x)} + \ve{\nb V(x)} \leq (\alpha - 1)L\ve{x} + \ve{\nb V(x)}.
\end{align*}
The result follows from combining the three inequalities above.
\end{proof}

\begin{lem}\label{l:perturb_error}
In the setting of Lemma~\ref{d_chi2}, we have for $t\in[kh, (k+1)h]$,
\begin{align*}
    &\E\ba{\psi_t(z_t)\ve{\nb\ln p_{kh}(z_t) - \nb\ln p_{t}(z_t)}^2}\\
    &\leq 
    \Gkht 
    \cdot\ba{\fc{C_{t,L}}{\chi^2(q_t||p_t) + 1}
    \Gkht 
    \sE_{p_t}\pf{q_t}{p_t} + C_{d,L}},
\end{align*}
where
\begin{align}\label{constant:ctl}
\CtLeq 
\end{align}
and 
\begin{align}\label{constant:cdl}
\CdLeq.
\end{align}
\end{lem}
\begin{proof}
In both SMLD and DDPM models, we have the following relationship for $t\in[kh, (k+1)h]$:
\begin{align*}
    p_{kh} = (p_t)_\al* \ph_{\sigma^2}.
\end{align*}
where $p_\al(x)=\al^d p(\al x)$.
In SMLD, $\alpha=1$ and $\sigma^2 =\int_{kh}^t g(T-s)^2\, ds$, while in DDPM, $\alpha = e^{\fc12\int_{kh}^t g(T-s)^2\,ds}$ and $\sigma^2 = 1 - e^{-\int_{kh}^t g(T-s)^2\,ds}$. 
Now for SMLD, 
\begin{align*}
    &\E\ba{\psi_t(z_t)\ve{\nb\ln p_{kh}(z_t) - \nb\ln p_{t}(z_t)}^2}\\
    & \leq 72L^2\sigma^2d + 8L^2\sigma^4 \E\ba{\psi_t(z_t)\ve{\nb\ln p_t(z_t)}^2} & \text{by Lemma~\ref{l:perturb_SMLD}}\\
    & \leq 72\sigma^2L^2d +
    \fc{32L^2\sigma^4}{\chi^2(q_t||p_t) + 1}\cdot
    \sE_{p_t}\pf{q_t}{p_t} +
    16\sigma^4L^3d& \text{by Corollary~\ref{l:DV_nb_ln_p}}\\
    & \leq
    \Gkht^2 
    \pa{\fc{32L^2}{\chi^2(q_t||p_t)+1} + 16 L^3d}
    + \Gkht \cdot 72L^2d\\
    &\le 
    \Gkht^2 
    \fc{32L^2}{\chi^2(q_t||p_t)+1}
    + \Gkht \cdot 76L^2d,
\end{align*}
where in the last inequality we use the fact that $g$ is increasing, so that for $h\le\rc{4Lg(T-kh)^2}$,
\[
\Gkht \smf = \int_{kh}^t g(T-s)^2\,ds\cdot  \smf \le
h \cdot g(T-kh)^2 \cdot \smf \le \rc 4.
\]
Recall that to use Lemma~\ref{l:perturb_SMLD}, it suffices that $L\leq \rc{2\al^2\sigma^2}$, and so it suffices that $h\leq\fc{1}{4Lg(T-kh)^2}$ in SMLD.

For DDPM, 
observe that for $h\leq\fc{1}{4g(T-kh)^2}$,
\begin{align*}
    \alpha &\leq 1+\int_{kh}^t g(T-s)^2ds \leq 1+(t-kh)g(T-kh)^2\leq 1+\rc 4\\
    \sigma^2 &= 1 - e^{-\int_{kh}^t g(T-s)^2ds} \leq \int_{kh}^t g(T-s)^2ds \leq (t-kh)g(T-kh)^2 \le  \fc14.
\end{align*}
By Lemma~\ref{l:perturb_DDPM}, using the assumption that $L\geq 1$, we obtain
\begin{align*}
    &\ \ \ \E\ba{\psi_t(z_t)\ve{\nb\ln p_{kh}(z_t) - \nb\ln p_{t}(z_t)}^2} \\
    & \leq 72\alpha^4L^2\sigma^2d + 4(\alpha + 2\alpha^3L\sigma^2)^2(\al-1)^2L^2\E\ba{\psi(z_t)\ve{z_t}^2} \\
    &\ \ \ + 4(\alpha - 1 + 2\alpha^3L\sigma^2)^2\E\ba{\psi_t(z_t)\ve{\nb\ln p_t(z_t)}^2}\\
    &\le 
    72\alpha^4L^2\sigma^2d + 44L^2 \Gkht^2 \E\ba{\psi(z_t)\ve{z_t}^2}\\
    &\quad + 
    100L^2 \Gkht^2\E\ba{\psi_t(z_t)\ve{\nb\ln p_t(z_t)}^2}
    \\
    & \le 44 L^2 d \Gkht\\
    &\quad + 44 L^2 \ba{\fc{2C_t^2}{\chi^2(q_t||p_t)+1}\sE_{p_t}\pf{q_t}{p_t} + \rc 2 \E_{p_t}\ve{x}^2 + \rc 2 C_t} \Gkht^2\\
    &\quad + 100L^2 \ba{\fc{4}{\chi^2(q_t||p_t)+1}\sE_{p_t}\pf{q_t}{p_t} + 2dL} \Gkht^2\\
    & \le 
    L^2 \Gkht
    \Bigg[
        \Gkht\pf{88C_t^2+400}{\chi^2(q_t||p_t)+1} \sE_{p_t}\pf{q_t}{p_t}\\ 
        &\quad 
        + 44d + 
        \Gkht
        \pa{
        22(\E_{p_t}\ve{x}^2 + C_t) + 200Ld
        }
    \Bigg]  \\
    &\le 
    \Gkht\ba{
    \Gkht \fc{88C_t^2+400}{\chi^2(q_t||p_t)+1} \sE_{p_t}\pf{q_t}{p_t} + 6+94L^2d
    },
\end{align*}
where 
we used 
Lemma~\ref{second_moment} and Corollary~\ref{l:DV_nb_ln_p}.
Here, we use the assumption that $h\le \rc{4g(T-kh)^2(\E_{p_t}\ve{x}^2 + C_t)}$.
\end{proof}

\subsection{Auxiliary Lemmas}
\label{ss:aux}
In this section, we continue with bounding errors in the form of $\E_{\psi_t q_t}\ve{u(x)}^2$. However, we only decompose them into errors which we have already bounded in the previous two sections. The following two lemmas will be directly applied in the proof of Lemma~\ref{l:error_A} and Lemma~\ref{l:error_B}.

\begin{lem}\label{second_moment_s_theta}
With the setting of Lemma~\ref{d_chi2}, we have the following bound of the second moment of estimated score function with respect to $\psi_t q_t$:
\begin{align*}
    \E\ba{\psi_t(z_t)\ve{\score(z_t, T-kh)}^2} & \leq \fc{4C_{t,L}\Gkht + 8}{\chi^2(q_t||p_t)+1}\cdot\sE_{p_t}\pf{q_t}{p_t} + 4(\Mkh^2 + C_{d,L} +dL),
\end{align*}
where $C_{t, L}$ and $C_{d,L}$ are constants defined in Lemma~\ref{l:perturb_error}.
\end{lem}
\begin{proof}
Note that by the triangle inequality,
\begin{align*}
    \ve{\score(x, T-kh)} & \leq \ve{\score(x, T-kh) - \nb\ln \Tilde{p}_{T-kh}(x)} \\ & + \ve{\nb\ln \Tilde{p}_{T-kh}(x) - \nb\ln \Tilde{p}_{T-t}(x)} + \ve{\nb\ln \Tilde{p}_{T-t}(x)},
\end{align*}
and hence,
\begin{align*}
   \ve{\score(x, T-kh)}^2 & \leq 4\ve{\score(x, T-kh) - \nb\ln \Tilde{p}_{T-kh}(x)}^2 
   \\ & + 4\ve{\nb\ln \Tilde{p}_{T-kh}(x) - \nb\ln \Tilde{p}_{T-t}(x)}^2 + 2\ve{\nb\ln \Tilde{p}_{T-t}(x)}^2.
\end{align*}
Recall that we need to bound this second moment of estimated score function with respect to $\psi_t q_T$. For the first term, as $\ve{\score(x. T-kh) -\nb\ln p_{kh}(x)} $ is $\Mkh$-bounded, we have trivial bound that 
\begin{align*}
    \E_{\psi_t q_t}\ve{\score(x, T-kh) - \nb\ln \Tilde{p}_{T-kh}(x)}^2\leq \Mkh^2.
\end{align*}
By Lemma~\ref{l:perturb_error}, the second term is bounded by
\begin{align*}
    &\E_{\psi_t q_t}\ba{\ve{\nb\ln p_{kh}(z_t) - \nb\ln p_{t}(z_t)}^2}\\
    &\leq 
    \Gkht
    \cdot\ba{\fc{C_{t,L}}{\chi^2(q_t||p_t) + 1}
    \Gkht
    \sE_{p_t}\pf{q_t}{p_t} + C_{d,L}}
\end{align*}
for constant $C_{t,L}$ and $C_{d, L}$ defined in~\eqref{constant:ctl} and~\eqref{constant:cdl} respectively. The last term is bounded in Corollary~\ref{l:DV_nb_ln_p} by
\begin{align*}
    \E\ba{\psi_t(z_t)\ve{\nb \ln p_t(z_t)}^2} &\leq
    \fc{4}{\chi^2(q_t||p_t)+1} \sE_{p_t}\pf{q_t}{p_t}+2dL.
\end{align*}
Combining these three inequalities, we obtain that for $h\leq \fc{1}{g(T-kh)^2}$,
\begin{align*}
    &\E\ba{\psi_t(z_t)\ve{\score(z_t, T-kh)}^2} \\
    & \leq \fc{4C_{t,L} + 8}{\chi^2(q_t||p_t)+1}
    \Gkht
    \sE_{p_t}\pf{q_t}{p_t} + 4(\Mkh^2 + C_{d,L} +dL). \qedhere
\end{align*}
\end{proof}


Now we bound $\E\ba{\psi_t(z_t)\ve{z_t-z_{kh}}^2}$.
\begin{lem}\label{zt_zkh_difference}
In the setting of Lemma~\ref{d_chi2}, if 
\[h\le 
\rc{g(T-kh)^2(8L^2+20L+3L_s+10C_t + \E_{p_t}\ve{x}^2)},\] then
\begin{align*}
    \E\ba{\psi_t(z_t)\ve{z_t-z_{kh}}^2} & \leq \fc92 
    \Gkht
    \ba{\fc{\Tilde{R}_{t}}{\chi^2(q_t||p_t)+1}\cdot\sE_{p_t}\pf{q_t}{p_t} + R_{t, kh} },
\end{align*}
where $\Tilde{R}_t$ and $\Rtkh $ are defined in~\eqref{R_t} and~\eqref{R_t_kh} respectively.
\end{lem}
\begin{proof}
Note that 
\begin{align*}
    &\ve{z_t-z_{kh}}\\
    & =\ve{\Gkht \score(z_{kh}, T-kh) - \int_{kh}^t f(z_{kh}, T-s)ds + \int_{kh}^t g(T-s)dw_s}\\
    & \leq
    \Gkht\ve{\score(z_{kh}, T-kh)} +  \fc12\ve{z_{kh}\int_{kh}^t g(T-s)^2ds} + \ve{\int_{kh}^t g(T-s)dw_s}\\
    & \leq \Gkht\ba{\ve{\score(z_{kh}, T-kh)} +\fc12  \ve{z_{kh}}} + \ve{\int_{kh}^t g(T-s)dw_s}\\
    & \leq \Gkht\ba{\ve{\score(z_t, T-kh)} + \sms\ve{z_t-z_{kh}} +\fc12 \ve{z_t} +\fc12 \ve{z_t - z_{kh}}} + \ve{\int_{kh}^t g(T-s)dw_s}\\
    & = \Gkht\ba{\ve{\score(z_t, T-kh)}  +\fc12 \ve{z_t}} + \bigg(\sms+\fc12\bigg) g(T-kh)^2\cdot h \ve{z_t-z_{kh}}+ \ve{\int_{kh}^t g(T-s)dw_s},
\end{align*}
where the next-to-last line is due to the fact that the estimated score function is $\sms$-Lipschitz. We also use the fact that $g(t)$ is an increasing function and hence $g(T-t)\leq g(T-kh)$ for any $t\in[kh, (k+1)h]$. Hence if $h\le  \fc{1}{3(\sms+1/2)g(T-kh)^2}$, then
\begin{align*}
    \ve{z_t-z_{kh}}\leq \fc32 
    \Gkht
    \ba{\ve{\score(z_t, T-kh)} + \fc12\ve{z_t}} + \fc32 \ve{\int_{kh}^t g(T-s)dw_s}.
\end{align*}
Therefore, by the fact that $(a+b)^2 \leq 2a^2 + 2b^2$ for any $a,b>0$,
\begin{align}\label{zt_zkh_diff}
   \ve{z_t-z_{kh}}^2 & \leq \fc92 
   \Gkht^2
   \ba{2\ve{\score(z_t, T-kh)}^2 + \fc12\ve{z_t}^2} + \fc92 \ve{\int_{kh}^t g(T-s)dw_s}^2. 
\end{align}
With the results of Lemma~\ref{second_moment_s_theta} and Lemma~\ref{second_moment}, we have
\begin{align*}
    &2\E\ba{\psi_t(z_t)\ve{\score(z_t, T-kh)}^2} + \fc12\E\ba{\psi_t(z_t)\ve{z_t}^2} \\
    & \leq \fc{ 8C_{t,L}\Gkht+C_t^2 +16}{\chi^2(q_t||p_t)+1}\cdot\sE_{p_t}\pf{q_t}{p_t} + 8(\Mkh^2 + C_{d,L} +dL) + \fc14 \E_{p_t}\ve{x}^2 + \fc14 C_t.
\end{align*}
Now plugging this and the result of Lemma~\ref{second_moment_diffusion} into~\eqref{zt_zkh_diff}, we get that
\begin{align*}
    &\E\ba{\psi_t(z_t)\ve{z_t-z_{kh}}^2}  \leq \fc{9}{2}\Gkht^2\cdot 8(\Mkh^2 + C_{d,L} +dL) \\ &\ \ \ + \fc{9}{2}\Gkht^2\cdot\ba{\fc{8C_{t,L}\Gkht+C_t^2 +16}{\chi^2(q_t||p_t)+1}\cdot\sE_{p_t}\pf{q_t}{p_t} + \fc14 \E_{p_t}\ve{x}^2 + \fc14 C_t}\\
    &\ \ \ + 9\Gkht \cdot\ba{\fc{8C_t}{\chi^2(q_t||p_t)+1}\cdot\sE_{p_t}\pf{q_t}{p_t} + d + 8\ln 2}.
\end{align*}
Hence, using the assumption on $h$, 
\begin{align*}
    \E\ba{\psi_t(z_t)\ve{z_t-z_{kh}}^2} & \leq \fc92 \Gkht\ba{\fc{
    K_1
    }{\chi^2(q_t||p_t)+1}\cdot\sE_{p_t}\pf{q_t}{p_t} + 
    K_2},
\end{align*}
where 
\begin{align*}
    K_1 :&= 8C_{t,L} \Gkht^2
    +(C_t^2+16) \Gkht + 16C_t\\
    &\le 
    8(88C_t^2 + 400L^2) 
    \rc{400L+100C_t^2}
    + (C_t^2+16) \rc{20L+10C_t} + 8C_t\\
    &\le 8+C_t+1+8C_t = 9(C_t+1)
\end{align*}
and
\begin{align*}
    K_2 :&= \ba{\fc14(\E_{p_t}\ve{x}^2 +C_t)+ 8(\Mkh^2 + C_{d,L}+dL)}\Gkht  + 2d + 16\ln 2\\
    &\le \ba{\fc14(\E_{p_t}\ve{x}^2 +C_t)+ 8(\Mkh^2 + 256L^2d+dL)}
    \pa{\rc{\E_{p_t}\ve{x}^2 +C_t+8L^2}} 
    + 2d + 16\ln 2\\
    &\le \rc 4 + 300d + 16\ln 2
    \le 300d + 12.
\end{align*}
Hence the lemma holds by setting
\begin{align}
\tildeRteq,\label{R_t}\\
\Rtkheq \label{R_t_kh}.
\end{align}
\end{proof}

\subsection{Proof of Theorem~\ref{t:p}}
\label{ss:p-pf}

We state a more precise version of Theorem~\ref{t:p}. The structure of the proof is similar to that of Theorem~\ref{t:corrector-tv-chi2}.

\begin{thm}[Predictor with $L^2$-accurate score estimate, DDPM]
\label{t:p-precise}
Let $p_{\textup{data}}:\R^d\to \R$ be a  probability density satisfying Assumption~\ref{a:p} with $M_2=O(d)$, 
and let $\wt p_t$ be the distribution resulting from evolving the forward SDE according to
\hyperref[i:DDPM]{DDPM} with $g\equiv 1$.
Suppose furthermore that 
$\nb\ln \wt p_t$ is $L$-Lipschitz for every $t\ge 0$, and that each $s(\cdot, t)$ satisfies Assumption~\ref{a:score}. 
Then if
\[
\ep = O\pf{\etv\echi^3}{(\CLS+d)\CLS^{5/2}(\smf\vee \sms)^2 
(\ln(\CLS d) \vee \CLS \ln(1/\etv^2))
},
\]
running~\eqref{e:P} starting from $\ppr$ for time $T=\Te\pa{\ln(\CLS d) \vee \CLS \ln\prc{\etv}}$
and step size $h=\Te\pf{\ep_\chi^2}{\CLS(\CLS+d)(\smf\vee \sms)^2}$ 
results in a distribution $q_T$ such that $q_T$ is $\etv$-far in TV distance from a distribution $\ol q_T$, where $\ol q_T$ satisfies $\chi^2(\ol q_T||p_{\textup{data}})\le\echi^2$. In particular, taking $\echi=\etv$, we have $\TV(q_T||\pdata)\le 2\etv$. 
\end{thm}

\begin{proof}[Proof of Theorem~\ref{t:p-precise}]
We first define the bad sets where the error in the score estimate is large,
\begin{align}
B_t:&=\bc{\ve{\nb \ln p_{t}(x)-s(x, T-t)}>\ep_1}
\end{align}
for some $\ep_1$ to be chosen. 

Given $t\ge 0$, let $t_- = h\ff th$. Given a bad set $B$, define the interpolated process by
\begin{align}
\label{e:interp-c}
    d\ol z_t &= 
    -\ba{
        f(z_{t_-}, T-t) - g(T-t)^2 b(z_{kh}, T-kh)
    }\,dt + g(T-t)\,dw_t,\\
    \nonumber
    \text{where }
    b(z,t) &= \begin{cases}
    s(z,t),& z\nin B_t\\
    \nb \ln p_t(z), &z\in B_t
    \end{cases}.
\end{align}
In other words, simulate the reverse SDE using the score estimate as long as the point is in the
good set (for the current $p_t$) at the previous discretization step, and otherwise use the actual gradient $\nb \ln p_t$. Let $\ol q_t$ denote the distribution of $\ol z_t$ when $\ol z_0\sim q_0$; note that $q_{nh}$ is the distribution resulting from running LMC with estimate $b$ for $n$ steps and step size $h$. Note that this process is defined only for purposes of analysis, as we do not have access to $\nb \ln p_t$.

We can couple this process with the predictor algorithm using $s$ so that as long as $x_{mh}\nin B_{mh}$, the processes agree, thus satisfying condition~\ref{i:couple} of Theorem~\ref{t:framework}.

Then by Chebyshev's inequality,
\begin{align*}
    P(B_t) \le \pf{\ep}{\ep_1}^2=:\de.
\end{align*}
Let $T=Nh$, and let $K_\chi = \chi^2(q_0||p_0)$.
Then by Theorem~\ref{t:p-iy-simple}, 
\begin{align*}
\chi^2(\ol q_{kh}||p_{kh}) &= \exp\pa{-\fc{kh}{16\CLS}}\chi^2(q_0||p_0) + 
    O\pa{\CLS(\ep_1^2 + (\sms^2+\smf^2d)h)}\\
    &=\exp\pa{-\fc{kh}{4\CLS}}\chi^2(\mu_0||p) + O(1).
\end{align*}
For this to be bounded by $\ep_\chi^2$, it suffices for the terms to be bounded by  $\fc{\ep_\chi^2}2, \fc{\ep_\chi^2}4, \fc{\ep_\chi^2}4$; this is implied by
\begin{align*}
    T &\ge 32\CLS \ln\pf{2K_\chi}{\echi^2} =:T_{\min}
    \\
    h &= O\pf{\ep_\chi^2}{\CLS(\CLS+d)(\smf\vee \sms)^2}\\
    \ep_1 &= O\pf{\ep_\chi}{\sqrt{\CLS}}.
\end{align*}
(We choose $h$ so that the condition in Theorem~\ref{t:p-iy-simple} is satisfied; note $\ep_\chi\le 1$.)
By Theorem~\ref{t:framework}, 
\begin{align*}
\nonumber
\TV(q_{nh},\ol q_{nh}) 
&\le 
\sum_{k=0}^{n-1} (1+\chi^2(q_{kh}||p))^{1/2}P(B_{kh})^{1/2}\\
\nonumber
&\le 
\pa{\sumz k{n-1} \exp\pa{-\fc{kh}{32\CLS}}\chi^2(q_0||p)^{1/2} + O(1)} \de^{1/2}\\
&\le 
\pa{\pa{\sumz k{\iy} \exp\pa{-\fc{kh}{32\CLS}}K_\chi} + O(n)}\fc{\ep}{\ep_1}\\
&\le 
\fc{\ep}{\ep_1}\pa{\fc{64\CLS}hK_\chi + O(n)}.
\end{align*}
In order for this to be $\le \ep_{\TV}$, it suffices for
\begin{align*}
    \ep &\le \ep_1\ep_{\TV}\cdot 
    O\pa{\rc{n} \wedge \fc{h}{\CLS K_\chi}
    }.
\end{align*}
Supposing that we run for time $T = \Te(T_{\min})$, we have that $n=\fc{T}h= O\pf{C_TT_{\min}}h$. Thus it suffices for 
\begin{align*}
    \ep &= \ep_1\ep_{\TV}\cdot O \pa{\fc{h}{T_{\min}}\wedge \fc{h}{32\CLS K_\chi}}\\
    &= 
    O\pa{\fc{\echi}{\sqrt{\CLS}} \cdot \etv \cdot \fc{\echi^2}{\CLS(\CLS+d)(\smf\vee \sms)^2} \pa{\rc{\CLS \ln (2K_\chi/\ep_\chi^2)}\wedge \rc{\CLS K_\chi} }
    }  \\
    &= 
    O\pf{\ep_{\TV}\ep_\chi^3}{\CLS^{5/2}(\CLS+d)(\smf\vee \sms)^2 (\ln (2K_\chi/\ep_\chi^2)\vee K_\chi) }.
\end{align*}
Finally, note that for $T=\Om( \ln (\CLS d))$, we have $K_\chi=O(1)$ by Lemma~\ref{l:warm}. Substituting $K_\chi=O(1)$ then gives the desired bound. 
\end{proof}

\subsection{Proof of Theorem~\ref{t:pc}}
\label{ss:pc-pf}

We now prove the main theorem on the predictor-corrector algorithm with $L^2$-accurate score estimate.


\pctheorem*

For simplicity, we consider the predictor-corrector algorithm in the case where all the corrector steps are at the end (but see the discussion following the proof for the general case).
The result will follow from chaining together the guarantee on the predictor algorithm (Theorem~\ref{t:p-precise}) and LMC (Theorem~\ref{t:corrector-tv-chi2}).
\begin{proof}[Proof of Theorem~\ref{t:pc}]
Let $M=T/h$. 
We take 
$h=\Te \prc{(\smf\vee \sms)^2\CLS (\CLS+d)}$, number of corrector steps 
$N_0= \cdots = N_{T/h-1} = 0$ and $N_{M}=T_{\text{c}}/h_M$, where $T_{\text{c}}=\Te\pa{\CLS \ln \pf{2}{\echi^2}}$ and $h_M = \Te\pf{\echi^2}{dL^2\CLS}$. 
Let the distribution of $z_{T,0}$ be $q_{T,0}$. 
By Theorem~\ref{t:p-precise}, if $T=\Te(\ln(\CLS d) \vee \CLS \ln (1/\etv))$, then 
\[
\ep = O\pf{\etv}{(\smf\vee \sms)^2(\CLS+d)\CLS^{5/2}(\ln (\CLS d) \vee \CLS \ln (1/\etv))},  
\]
then there exists $\ol q_{T,0}$ such that $\TV(q_{T,0}, \ol q_{T,0}) = \etv/2$ and $\chi^2(\ol q_{T,0}||\pdata)=1$. Then using Theorem~\ref{t:corrector-tv-chi2} with $\etv\leftarrow \etv/2$ and $K_\chi = 1$, plus the triangle inequality gives that if 
\[
\ep = O\pf{\etv\echi^3}{dL^2\CLS^{5/2}\ln (1/\etv)},
\]
then there is $\ol q_T$ such that $\TV(q_T, \ol q_T)=\etv$ and $\chi^2(\ol q_T||\pdata) = \echi^2$. Finally, setting $\etv, \echi \leftarrow \etv/2$ gives $\TV(q_T,\pdata)\le \etv$. 

We note that for $\etv^3 = O\pf{1}{(1+\sms/\smf)^2(1+\CLS/d)(\ln (\CLS d)\vee \CLS)}$, the second condition on $\ep$ is more constraining, giving the theorem.
\end{proof}
\begin{rem*}
We can also analyze a setting where predictor and corrector steps are interleaved; for instance, if $N=1$, then interleaving the one-step inequalities in Theorem~\ref{t:lmc-iy} and~\ref{t:p-iy-simple} gives a recurrence
\begin{align*}
    \chi^2(q_{(k+1)h,0}||p_{(k+1)h}) &\le \exp\pa{-\fc{h_{\textup{pred}}}{16\CLS}}  
    \chi^2(q_{kh,1}||p_{kh}) + O(dL^2h^2 + \ep_1^2h)\\
    \chi^2(q_{(k+1)h,1} || p_{(k+1)h)}) &\le \exp\pa{-\fc{h_{\textup{corr}}}{4\CLS}} \chi^2(q_{(k+1)h,0}||p_{(k+1)h}) + O (\ep_1^2h + (\sms^2+\smf^2d)h^2);
\end{align*}
we can then follow the proof of Theorem~\ref{t:p}. While this does not improve the parameter dependence under the assumptions of  Theorem~\ref{t:pc}, it can potentially allow for larger step sizes (beyond what is allowed by Theorem~\ref{t:p}), as error accrued in the predictor step can be exponentially damped by the corrector step.
\end{rem*}
\section{Stationary distribution of LD with score estimate can be arbitrarily far away}
\label{a:far}

We show that the stationary distribution of Langevin dynamics with $L^2$-accurate score estimate can be arbitrarily far from the true distribution. 
We can construct a counterexample even in one dimension, and take the true distribution as a standard Gaussian $p(x) = \frac{1}{\sqrt{2\pi}} e^{-x^2/2}$. We will take the score estimate to also be in the form $\nb \ln q$, so that the stationary distribution of LMC with the score estimate is $q$. The main idea of the construction is to set $q$ to disagree with $p$ only in the tail of $p$, where it has a large mode; this error will fail to be detected under $L^2(p)$.
\begin{thm}
Let $p$ be the density function of $N(0,1)$. There exists an absolute constant $C$ such that given any $\ep>0$, there exists a distribution $q$ such that 
\begin{enumerate}
    \item $\ln q$ is $C$-smooth.  
    \item 
    $\E_p[\|\nb \ln p - \nb \ln q\|^2] <\ep$
    \item 
    $\TV(p,q) > 1-\ep$.
\end{enumerate}
\end{thm}

\begin{proof}
Take a smooth non-negative function $g$ supported on $[-1, 1]$, with $\max \lvert g''\rvert \leq c$ and $g(0) = 1$. We consider a family of distributions for $L > 0$ with density
\begin{equation*}
    q_L(x) \propto e^{-V_L(x)}, \qquad \text{and} \quad V_L(x) := \frac{x^2}{2} - L^2 g\pa{\frac{2}{L}(x - L)}.
\end{equation*}
Thus the score function for $q_L$ is given by 
\begin{equation*}
    V_L'(x) = x - (2L) g'\Bigl( \frac{2}{L} (x - L) \Bigr).
\end{equation*}
We compute the $L^2(p)$ error between the score functions associated with $p$ and $q_L$.
\begin{equation*}
    \begin{aligned}
        \E_p (V_L'(x) - x)^2 &= \frac{1}{\sqrt{2\pi}} \iny (2L)^2 \Bigl\lvert g'\Bigl( \frac{2}{L} (x - L) \Bigr)\Bigr\rvert^2 e^{-x^2/2} \dx   \\
        & \leq \frac{1}{\sqrt{2\pi}} (2L)^2 e^{-L^2/8} \iny \Bigl\lvert g'\Bigl( \frac{2}{L} (x - L) \Bigr)\Bigr\rvert^2  \dx \\
        & = \frac{1}{\sqrt{2\pi}} 2 L^3 e^{-L^2/8} \iny \bigl\lvert g'(y)\bigr\rvert^2\,  dy,
    \end{aligned}
\end{equation*}
where in the first inequality we have used that $g(\frac{2}{L}(x - L))$ has support $[\frac{L}{2}, \frac{3L}{2}]$, since $g$ has support $[-1, 1]$. 
Thus the $L^2(p)$-error of the score function goes to $0$ as $L \to \infty$.

Moreover, as 
\begin{equation*}
    \bigl\lvert V_L''(x)\bigr\rvert = \ab{ 1 - 4 g''\Bigl( \frac{2}{L} (x - L) \Bigr) }
    \leq 1 + 4 \max_{y} \bigl\lvert g''(y) \bigr\rvert \leq 1+4c, 
\end{equation*}
the distribution $q_L$ satisfies the required smoothness (Lipschitz score) assumption. 
Note that $q_L$ has a large mode concentrated at $x = L$ as 
\begin{equation*}
    V_L(L) = \frac{L^2}{2} - L^2 g(0) = - \frac{L^2}{2}, 
\end{equation*}
while $p$ has vanishing density there, which is in fact the reason that $L^2(p)$-loss of the score estimate is not able to detect the difference between the two distributions. As the height (and width) of the mode becomes arbitrarily large compared to $x=0$, we have $q_L([\fc L2, \fc{3L}2]) \to 1$, whereas $p_L([\fc L2, \fc{3L}2])\to 0$. Hence $\TV(p_L,q_L)\to 1$.
\end{proof}

\section{Useful facts}
In this section, we collect some facts and technical lemmas used throughout the paper.
\subsection{Facts about probability distributions}
\label{s:facts}
Given a probability measure $P$ on $\R^d$ with density $p$, we say that a Poincar\'e inequality (PI) holds with constant $\CP$ if for any probability measure $q$,
\begin{align}
\chi^2(q||p)&\leq \CP \sE_p\pf{q}{p}:= \CP\int_{\R^d}\ve{\nb\fc{q(x)}{p(x)}}^2p(x) dx.\tag{PI}
\end{align}
Alternatively, for any $C^1$ function $f$,
\begin{align*}
    \Var_p(f) &\le \CP \int_{\R^d} \ve{\nb f}^2 p(x)\dx.
\end{align*}
We say that a log-Sobolev inequality (LSI) holds with constant $\CLS$ if for any probability measure $q$,
\begin{align}
    \KL(q||p) \leq 
    \fc{\CLS}2\int_{\R^d}\ve{\nb\ln\fc{q(x)}{p(x)}}^2\blu{q}(x)dx.\tag{LSI}
\end{align}
We call the Poincar\'e constant and log-Sobolev constant the smallest $\CP$, $\CLS$ for which the inequalities hold for all $q$. If $p$ satisfies a log-Sobolev inequality with constant, then $p$ satisfies a Poincar\'e inequality with the same constant; hence the Poincar\'e constant is at most the log-Sobolev constant, $\CP\le \CLS$. If $p\propto e^{-V}$ is $\al$-strongly log-concave, that is, $V\succeq \al I_d$, then $p$ satisfies a log-Sobolev inequality with constant $1/\al$. 

We collect some properties of distributions satisfying LSI or PI.


\begin{lem}[Herbst, Sub-exponential and sub-gaussian concentration given log-Sobolev inequality, {\cite[Pr. 5.4.1]{bakry2013analysis}}]
\label{l:herbst}
Suppose that $\mu$ satisfies a log-Sobolev inequality with constant $\CLS$. Let $f$ be a 1-Lipschitz function.
Then
\begin{enumerate}
\item
(Sub-exponential concentration) For any $t\in \R$, 
\[
\E_\mu e^{tf} \le e^{t\E_\mu f + \fc{\CLS t^2}2}.
\]
\item
(Sub-gaussian concentration) For any $t\in \left[0,\rc{\CLS}\right)$, 
\[
\E_\mu e^{\fc{tf^2}2} \le
\rc{\sqrt{1-\CLS t}}\exp\ba{\fc{t}{2(1-\CLS t)} (\E_\mu f)^2}.
\]
\end{enumerate}
\end{lem}

\begin{lem}[Gaussian measure concentration for LSI, {\cite[\S5.4.2]{bakry2013analysis}}]
\label{l:gmc-lsi}
Suppose that $\mu$ satisfies a log-Sobolev inequality with constant $\CLS$. Let $f$ be a $L$-Lipschitz function. Then
\[
\mu\pa{\ab{f - \E_\mu f} \ge r} \le 2e^{-\fc{r^2}{2\CLS L^2}}.
\]
\end{lem}


\begin{lem}[{\cite[Lemma G.10]{ge2018beyond}}]\label{lem:mode-mean}
\label{l:mean-mode}
Let $V:\R^d\to \R$ be a $\al$-strongly convex and $\be$-smooth function and let $P$ be a probability measure with density function $p(x)\propto e^{-V(x)}$. Let $x^*=\amin_x V(x)$ and $\ol x = \E_P x$. Then 
\begin{align}
\ve{x^*-\ol x} &\le \sfc{d}{\al}\pa{\sqrt{\ln \pf{\be}{\al}}+5}.
\end{align}
\end{lem}

\begin{thm}[{ \cite[Theorem 5.1]{brascamp2002extensions}, \cite{harge2004convex}}]\label{thm:harge}
\label{t:bl}
Suppose the $d$-dimensional gaussian $N(0,\Si)$ has density $\ga$. Let $p=h\cdot \ga$ be a probability density.
\begin{enumerate}
    \item If $h$ is log-concave, and $g$ is convex, then 
    \begin{align*}
\int_{\R^d} g(x-\E_p x)p(x)\dx &\le \int_{\R^d} g(x)\ga(x)\dx.
\end{align*}
    \item If $h$ is log-convex,\footnote{Note that the sign is flipped in the theorem statement in \cite{brascamp2002extensions}.} and $g(x) = \an{x,y}^\al$ for some $y\in \R^d$, $\al>0$, then 
    \begin{align*}
\int_{\R^d} g(x-\E_p x)p(x)\dx &\ge \int_{\R^d} g(x)\ga(x)\dx.
\end{align*}
\end{enumerate}
\end{thm}

\begin{lem}\label{l:lc}
Let $P$ be a probability measure on $\R^d$ with density function $p$ 
such that $\ln p$ is $C^1$ and $L$-smooth and $P$ satisfies a Poincar\'e inequality with constant $\CP$. Then $L\CP\ge 1$.
\end{lem}
\begin{proof}
By the Poincar\'e inequality and 
Lemma~\ref{t:bl}(2), 
since $p$ is equal to the density of $N(0,\rc L I_d)$ multiplied by a log-convex function, 
\begin{align*}
    \CP &\ge \E_P(x_1 - \E_P x_1)^2 \ge \E_{N(0,\rc{L}I_d)} x_1^2 = \rc{L}. \qedhere
\end{align*}
\end{proof}

\subsection{Lemmas on SMLD and DDPM}

We give bounds on several quantities associated with the SMLD and DDPM processes at time $t$: the log-Sobolev constants (Lemma~\ref{l:lsi-noise}), the second moment (Lemma~\ref{2ed_moment}), and the warm start parameter (Lemma~\ref{l:warm}).

First, we note that for SMLD and DDPM, the conditional distribution of $\wt x_t$ given $\wt x_0$ is
\begin{align*}
\text{SMLD:} &&\quad 
    \wt x_t|\wt x_0 &\sim N\pa{x(0), \int_0^t g(s)^2 \,ds \cdot I_d} \\
\text{DDPM:}&& \quad
    \wt x_t|\wt x_0 &\sim N\pa{x(0)e^{-\rc 2 \int_0^t g(s)^2\,ds}, (1-e^{-\int_0^t g(s)^2\,ds})I_d}.
\end{align*}
Hence
\begin{align}
\label{e:p-SMLD}
    \wt p^{\textup{SMLD}}_t &= p_0 * N\pa{0,\int_0^t g(s)^2 \,ds \cdot I_d}\\
\label{e:p-DDPM}
    \wt p^{\textup{DDPM}}_t  &=
    M_{e^{-\rc 2 \int_0^t g(s)^2\,ds}\#}
    p_0 * N(0,(1-e^{-\int_0^t g(s)^2\,ds})I_d)
\end{align}
where $M_c$ is multiplication by $c$.

\begin{lem}[\cite{chafai2004entropies}]
\label{l:lsi-conv}
Let $p,p'$ be two probability densities on $\R^d$. 
If $p$ and $p'$ satisfy log-Sobolev inequalities with constants $\CLS$ and $\CLS'$, then $p*p'$ satisfies a log-Sobolev inequality with constant $\CLS+\CLS'$.
\end{lem}
\begin{lem}[Log-Sobolev constant for SMLD and DDPM]\label{l:lsi-noise}
Let $\wt p^{\textup{SMLD}}_t$ and $\wt p^{\textup{DDPM}}_t$ denote the distribution of the SMLD/DDPM processes at time $t$, when started at $p_0$. Let $\CLS$ be the log-Sobolev constant of $p_0$. Then 
\begin{align*}
    \CLS(\wt p^{\textup{SMLD}}_t) &\le \CLS + \int_0^t g(s)^2 \,ds \\
    \CLS(\wt p^{\textup{DDPM}}_t) &\le (\CLS-1) e^{- \int_0^t g(s)^2 \,ds} + 1 \le \max\{\CLS,1\}.
\end{align*}
\end{lem}
Note that the analogous statement for the Poincar\'e constant $\CP$ holds for Lemma~\ref{l:lsi-conv} and~\ref{l:lsi-noise}.
\begin{proof} 
Note that if $\mu$ has log-Sobolev constant $\CLS$ and $T$ is a smooth $L$-Lipschitz map, then $T_\# \mu$ has log-Sobolev constant $\le L^2\CLS$. Applying Lemma~\ref{l:lsi-conv} to~\eqref{e:p-SMLD} and~\eqref{e:p-DDPM} then finishes the proof.
\end{proof}

\begin{lem}[Second moment for SMLD and DDPM]\label{2ed_moment}
Suppose that $\Tilde{p}_0$ has finite second moment, then for $t\in[0,T]$:
\begin{align*}
    & \E_{\Tilde{p}_t}\ba{\ve{x}^2} = \E_{\Tilde{p}_0}\ba{\ve{x}^2} + d\beta(t) &&\textit{in SMLD,}\\
    & \E_{\Tilde{p}_t}\ba{\ve{x}^2} = e^{-\beta(t)}\E_{\Tilde{p}_0}\ba{\ve{x}^2} + d(1-e^{-\beta(t)})\le \max\bc{\E_{\Tilde{p}_0}\ba{\ve{x}^2}, d} &&\textit{in DDPM,}
\end{align*}
where $\beta(t) = \int_0^t g(s)^2ds$.
\end{lem}
\begin{proof}
Recall that in SMLD, $\Tilde{x}_t\sim N(\Tilde{x}_0, \beta(t) \cdot I_d)$. Let $y\sim N(0, \beta(t) \cdot I_d)$ be independent of $\Tilde{x}_0$. Then
\begin{align*}
    \E_{\Tilde{p}_t}\ba{\ve{x}^2} & = \E\ba{\ve{\Tilde{x}_0+y}^2} = \E\ba{\ve{\Tilde{x}_0}^2} + \E\ba{\ve{y}^2} = \E\ba{\ve{\Tilde{x}_0}^2} + d\beta(t).
\end{align*}
In DDMP, $\Tilde{x}_t\sim N(e^{-\fc12\beta(t)}\Tilde{x}_0, (1-e^{-\beta(t)})\cdot I_d)$. Choose $y\sim N(0, (1-e^{-\beta(t)})\cdot I_d)$ independent of $\Tilde{x}_0$, then
\begin{align*}
    \E_{\Tilde{p}_t}\ba{\ve{x}^2} & = \E\ba{\ve{e^{-\fc12\beta(t)}\Tilde{x}_0+y}^2} = \E\ba{\ve{e^{-\fc12\beta(t)}\Tilde{x}_0}^2} + \E\ba{\ve{y}^2}\\
    &= e^{-\beta(t)}\E\ba{\ve{\Tilde{x}_0}^2} + d(1-e^{-\beta(t)}). \qedhere
\end{align*}

\end{proof}

\begin{lem}[Warm start for SMLD and DDPM]
\label{l:warm}
Suppose that $p$ 
has log-Sobolev constant at most $\CLS$ and 
$\ve{\E_{y\sim p} y}\le M_1$. 
Let $\ph_{\si^2}$ denote the density of $N(0,\si^2I_d)$. 
Then for any $\si^2$, 
\[
\chi^2(\ph_{\si^2}|| p*\ph_{\si^2})\le 4\exp\pa{\fc{d(2M_1 + 8\CLS)}{\si^2}}
\]
Hence, letting $\si_{\textup{SMLD}}^2 =
\int_0^t g(s)^2\,ds
$ and $\si_{\textup{DDPM}}^2 = 1-e^{-\int_0^t g(s)^2\,ds}$,
\begin{align*}
    \chi^2(\ph_{\si_{\textup{SMLD}}^2}||\wt p_t^{\textup{SMLD}})
    &\le 4\exp\pf{d(2M_1 + 8\CLS)}{\si_{\textup{SMLD}}^2} \\
    \chi^2(\ph_{\si_{\textup{DDPM}}^2}||\wt p_t^{\textup{DDPM}})
    &\le 4\exp\pf{d\pa{2e^{-\rc 2 \int_0^t g(s)^2\,ds}M_1+8e^{- \int_0^t g(s)^2\,ds}\CLS }}{\si_{\textup{DDPM}}^2} .
\end{align*}
\end{lem}
\begin{proof}
Let $R_x=(M_1 + 2\sqrt{\CLS})\ve{x}$. For a fixed $x$, note that $\E_{y\sim p}\an{y,x} \le \ve{\E_{y\sim p} y}\ve{x}\le M_1\ve{x}$ by assumption. 
Then by Lemma~\ref{l:gmc-lsi},
\begin{align*}
    \Pj(\an{y,x} \ge R_x) &\le \Pj(|\an{y,x}-\E_{y\sim p}\an{y,x}|\ge 2\sqrt{\CLS}\ve{x})
    \le 2e^{-\fc{\pa{2\sqrt{\CLS}\ve{x}}^2}{2\CLS \ve{x}^2}} \le 2e^{-2}\le \rc 2.  
\end{align*}
Hence
\begin{align*}
    (p*\ph_{\si^2})(x) 
    &= \prc{2\pi \si^2}^{d/2}\int_{\R^d}e^{-\fc{\ve{x+y}^2}{2\si^2}} p(y)\,dy\\
    &\ge \prc{2\pi \si^2}^{d/2}e^{-\fc{\ve{x}^2}{2\si^2}}\int_{\R^d} e^{-\fc{\an{x,y}}{\si^2}} p(y)\,dy\\
    &\ge 
    \prc{2\pi \si^2}^{d/2}e^{-\fc{\ve{x}^2}{2\si^2}} \int_{\an{y,x}\le R_x} e^{-\fc{\an{x,y}}{\si^2}} p(y)\,dy\\
    &\ge 
    \prc{2\pi \si^2}^{d/2}e^{-\fc{\ve{x}^2}{2\si^2}} \int_{\an{y,x}\le R_x} e^{-(M_1\ve{x} + 2 \sqrt{\CLS}\ve{x})/\si^2} p(y)\,dy\\
    &\ge \prc{2\pi \si^2}^{d/2}e^{-\fc{\ve{x}^2}{2\si^2}}
    e^{-\fc{\ve{x}^2}{8\si^2d} - \fc{2M_1^2d}{\si^2}-\fc{\ve{x}^2}{8\si^2d} - \fc{8\CLS d}{\si^2}}
    \int_{\an{y,x}\le R_x}  p(y)\,dy\\
&\ge    \prc{2\pi \si^2}^{d/2}e^{-\fc{\ve{x}^2}{2\si^2\pa{1-\rc{2d}}^{-1}}} e^{-\fc{d(8\CLS + 2M_1^2)}{\si^2}}\cdot \rc 2\\
    & \ge 
    \rc 2
    e^{-\fc{d(8\CLS + 2M_1^2)}{\si^2}}
    \pa{1-\rc{2d}}^{d/2}
    \ph_{\fc{\si^2}{1-\rc{2d}}}.
\end{align*}
Using the fact that $\chi^2(N(0,\Si_2)||N(0,\Si_1)) = \fc{|\Si_1|^{1/2}}{|\Si_2|} |(2\Si_2^{-1}-\Si_1^{-1})|^{-\rc 2}-1$, we have
\begin{align*}
    \chi^2(\ph_{\si^2}|| p*\ph_{\si^2}) + 1
    &\le 
    2\cdot e^{\fc{d(8\CLS + 2M_1^2)}{\si^2}}\pa{1-\rc{2d}}^{-\fc d2}
    \ba{\chi^2\pa{\ph_{\si^2}||\ph_{\fc{\si^2}{1-\rc{2d}}}}+1} \\
    &= 2\cdot e^{\fc{d(8\CLS + 2M_1^2)}{\si^2}}\pa{1-\rc{2d}}^{-\fc d2} \pa{1-\rc{2d}}^{-\fc d2} \pa{2-\pa{1-\rc{2d}}}^{-\fc d2}\\
    &\le 2\cdot e^{\fc{d(8\CLS + 2M_1^2)}{\si^2}}\pa{1-\rc{2d}}^{-d}
    \le 4e^{\fc{d(8\CLS + 2M_1^2)}{\si^2}}
\end{align*}
The corollary inequalities then follow from~\eqref{e:p-SMLD} and~\eqref{e:p-DDPM}, where for DDPM, we use the fact that  $M_{e^{-\rc 2 \int_0^t g(s)^2\,ds}\#}
    p_0$ has mean $e^{-\rc 2 \int_0^t g(s)^2\,ds}\cdot \E_p x$ and log-Sobolev constant $(e^{-\rc 2 \int_0^t g(s)^2\,ds})^2\CLS$.
\end{proof}





\end{document}